\documentclass[11pt]{article}

\usepackage{natbib}
\usepackage{graphicx}
\usepackage{amssymb}
\usepackage{amsmath}
\usepackage{bm}
\usepackage{amsthm}
\usepackage{amsfonts}
\usepackage{hyperref}
\usepackage{xr}
\usepackage{dsfont}
\usepackage{authblk}
\usepackage[margin=1in]{geometry}

\newcommand\blfootnote[1]{%
	\begingroup
	\renewcommand\thefootnote{}\begin{NoHyper}\footnote{#1}%
		\addtocounter{footnote}{-1}\end{NoHyper}%
	\endgroup
}

\usepackage{microtype}
\usepackage{graphicx}
\usepackage{subfigure}
\usepackage{booktabs} 

\usepackage{natbib}
\usepackage{comment}
\usepackage{graphicx}
\usepackage{caption}
\usepackage{amssymb}
\usepackage{amsmath}
\usepackage{graphics}
\usepackage{bm}
\usepackage{color}
\usepackage{amsthm}
\usepackage{enumerate}
\usepackage{amsfonts}
\usepackage{bbm}
\usepackage[toc]{appendix}
\usepackage{enumitem}
\usepackage{dsfont}
\usepackage{pdflscape}
\usepackage{float}
\usepackage{multirow}
\usepackage{xr}

\usepackage{tikz}
\usetikzlibrary{decorations.pathreplacing}

\newcounter{row}
\newcounter{col}

\renewcommand{\hat}{\widehat}
\def\singlespace{\def\baselinestretch{1}\@normalsize}

\def\wh{\widehat}
\def\wt{\widetilde}

\newcommand{\diag}{{\rm diag}}

\newcommand{\tr}{{\rm tr}}

\newcommand{\bA}{{\mathbf A}}
\newcommand{\bB}{{\mathbf B}}

\newcommand{\bH}{{\mathbf H}}

\newcommand{\bU}{{\mathbf U}}
\newcommand{\bV}{{\mathbf V}}

\newcommand{\bX}{{\mathbf X}}
\newcommand{\bY}{{\mathbf Y}}

\newcommand{\be}{{\mathbf e}}

\newcommand{\bu}{{\mathbf u}}
\newcommand{\bv}{{\mathbf v}}

\newcommand{\bx}{{\mathbf x}}

\newcommand{\bTheta} {\boldsymbol{\Theta}}

\def\6bullets{\bullet\bullet\bullet\bullet\bullet\bullet}

\newcommand{\mbE}{\mathbb{E}}
\newcommand{\mbR}{\mathbb{R}}
\newcommand{\mbP}{\mathbb{P}}
\newcommand{\pr}{\textsf{P}}
\newcommand{\MC}{MC}

\newcommand{\norm}[1]{\Vert#1\Vert}
\newcommand{\Norm}[1]{\left\Vert#1\right\Vert}
\newcommand{\inner}[2]{\langle #1, #2 \rangle}
\newcommand{\Inner}[2]{\left\langle #1, #2 \right\rangle}
\newcommand{\abs}[1]{\vert#1\vert}
\newcommand{\Abs}[1]{\left\vert#1\right\vert}

\newcommand{\Ind}[1]{\mathds{I}\left[#1\right]}

\newtheorem{lem}{Lemma}

\newtheorem{thm}{Theorem}

\newtheorem{prop}{Proposition}

\numberwithin{equation}{section}

\usepackage{hyperref}

\usepackage{algorithm}
\usepackage{algorithmic}

\title{Median Matrix Completion: from Embarrassment to Optimality}
\author[1]{Weidong Liu}
\author[2]{Xiaojun Mao}
\author[3]{Raymond K. W. Wong}
\affil[1]{School of Mathematical Sciences and MoE Key Lab of Artificial Intelligence Shanghai Jiao Tong University, Shanghai, 200240, China}
\affil[2]{School of Data Science, Fudan University, Shanghai, 200433, China}
\affil[3]{Department of Statistics, Texas A\&M University, College Station, TX 77843, U.S.A.}

\date{}

\begin{document}
\maketitle
\begin{abstract}
    In this paper, we consider matrix completion with absolute deviation loss
    and obtain an estimator of the median matrix.
    Despite several appealing properties of median,
    the non-smooth absolute deviation loss
    leads to computational challenge for
    large-scale data sets which are increasingly common among matrix completion
    problems. A simple solution to large-scale problems is parallel computing.
    However, embarrassingly parallel fashion often leads to inefficient
    estimators. Based on the idea of pseudo data,
    we propose a novel refinement step, which turns such inefficient
    estimators into a rate (near-)optimal matrix completion procedure.
    The refined estimator is an
    approximation of a regularized least median estimator,
    and therefore not an ordinary regularized empirical risk estimator.
    This leads to a non-standard analysis of asymptotic behaviors.
    Empirical results are also provided to confirm the effectiveness of the proposed method.
\end{abstract}
\blfootnote{E-mail addresses: weidongl@sjtu.edu.cn, maoxj@fudan.edu.cn, raywong@tamu.edu.}
\section{Introduction}

Matrix completion (\MC{}) has recently gained a substantial amount of popularity among
researchers and practitioners due to its wide applications;
as well as various related theoretical advances \citet{Candes-Recht09,Candes-Plan10,Koltchinskii-Lounici-Tsybakov11,Klopp14}.
Perhaps the most well-known example of a \MC{} problem
is the Netflix prize problem \citep{Bennett-Lanning07}, of which the goal is to predict missing entries of a partially observed matrix of movie ratings.
Two commonly shared challenges among \MC{} problems are high dimensionality of the matrix and a huge proportion of missing entries.
For instance, Netflix data has less than 1\% of observed entries of a matrix with around $5\times 10^5$ rows and $2\times 10^4$ customers.
With technological advances in data collection, we are confronting
increasingly large matrices nowadays.

Without any structural assumption on the target matrix, it is well-known that
\MC{} is an ill-posed problem.
A popular and often well-justified assumption is low rankness,
which however leads to a challenging and non-convex rank minimization problem
\cite{Srebro-Rennie-Jaakkola05}.
The seminal works of \citet{Candes-Recht09,Candes-Tao10,Gross11} showed that, when the entries are observed without noise, a perfect recovery of a low-rank matrix
can be achieved by a convex optimization via near minimal order of sample size, with high probability.
As for the noisy setting, some earlier work \citep{Candes-Plan10,Keshavan-Montanari-Oh10,Chen-Chi-Fan19} focused on arbitrary, not necessarily random, noise.
In general, the arbitrariness may prevent asymptotic recovery even in a probability sense.

Recently, a significant number of works
\citep[e.g.][]{Bach08,Koltchinskii-Lounici-Tsybakov11,Negahban-Wainwright11,Rohde-Tsybakov11, Negahban-Wainwright12,Klopp14,Cai-Zhou16,Fan-Gong-Zhu19,Xia-Yuan19}
targeted at more amenable
random error models,
under which (near-)optimal estimators had been proposed.
Among these work, trace regression model 
is one of the most popular models due to its regression formulation.
Assume $N$ independent pairs $(\bX_k,Y_k)$, for $k=1,\dots,N$, are observed,
where $\bX_k$'s are random design matrices
of dimension $n_1\times n_2$
and $Y_k$'s are response variables in $\mbR$.
The trace regression model assumes
\begin{equation}
Y_k=\tr\left(\bX_k^{\rm T}\bA_{\star}\right)+\epsilon_{k},\qquad k=1,\dots,N,
\label{eqn:tracereg}
\end{equation}
where $\tr(\bA)$ denotes the trace of a matrix $\bA$, and
$\bA_{\star}\in\mbR^{n_1\times n_2}$ is an unknown target matrix. Moreover,
the elements of
$\bm{\epsilon}=(\epsilon_1,\dots,\epsilon_N)$ are $N$ i.i.d.~random noise
variables independent of the design matrices.
In \MC{} setup, the design matrices $\bm{X}_k$'s are assumed to lie in
the set of canonical bases
\begin{equation}
\mathcal{X}=\{\bm{e}_j(n_1)\bm{e}_k(n_2)^{\rm T} : j=1,\dots, n_1; k=1,\dots, n_2\},
\label{eqn:canonical-basis}
\end{equation}
where
$\bm{e}_j(n_1)$ is the $j$-th unit vector in $\mathbb{R}^{n_1}$,
and 
$\bm{e}_k(n_2)$ is the $k$-th unit vector in $\mathbb{R}^{n_2}$.
Most methods then apply a regularized empirical risk minimization (ERM) framework
with a quadratic loss.
It is well-known that the quadratic loss
is most suitable for light-tailed (sub-Gaussian) error,
and leads to non-robust estimations.
In the era of big data,
a thorough and accurate data cleaning step, as part of data preprocessing, becomes virtually impossible.
In this regard,
one could argue that robust estimations are more desirable,
due to their reliable performances even in the presence of outliers
and violations of model assumptions.
While robust statistics is a well-studied area with a rich history \citep{Davies93,huber2011robust},
many robust methods were developed for small data by today's standards,
and are deemed too computationally intensive for big data or complex models.
This work can be treated as part of the general effort
to broaden the applicability of robust methods
to modern data problems.

\subsection{Related Work}

Many existing robust \MC{} methods
adopt regularized ERM and
assume observations are obtained from a low-rank-plus-sparse structure
$\bm{A}_{\star}+\bm{S} + \bm{E}$,
where the low-rank matrix $\bm{A}_{\star}$
is the target uncontaminated component; the sparse matrix $\bm{S}$
models the gross corruptions (outliers) locating at a small proportion of
entries;
and $\bm{E}$ is an optional (dense) noise component.
As gross corruptions are already taken into account,
many methods with low-rank-plus-sparse structure
are based on quadratic loss.
\citet{Chandrasekaran-Sanghavi-Parrilo11,Candes-Li-Ma11,Chen-Jalali-Sanghavi13,Li13} considered the
noiseless setting (i.e., no $\bm{E}$) with an element-wisely sparse $\bm{S}$.
\citet{Chen-Xu-Caramanis11}
studied the noiseless model with column-wisely sparse $\bm{S}$.
Under the model with element-wisely sparse $\bm{S}$,
\citet{Wong-Lee17} looked into the setting of arbitrary (not necessarily random)
noise $\bm{E}$,
while \citet{Klopp-Lounici-Tsybakov17} and \citet{Chen-Fan-Ma20}
studied random (sub-Gaussian) noise model for $\bm{E}$.
In particular, it was shown in Proposition 3 of \citet{Wong-Lee17}
that in the regularized ERM framework,
a quadratic loss with element-wise $\ell_1$ penalty on the sparse component
is equivalent to a direct application of a Huber loss without the sparse component.
Roughly speaking, this class of robust methods, based on the low-rank-plus-sparse
structure, can be understood as regularized ERMs with Huber loss.

Another class of robust \MC{} methods is based on
the absolute deviation loss, formally defined in \eqref{eqn:sto_opt}.
The minimizer of the corresponding risk has an interpretation of median (see Section \ref{sec:lad}),
and so the regularized ERM framework that applies absolute deviation loss
is coined as median matrix completion
\citep{Elsener-Geer18,Alquier-Cottet-Lecue19}.
In the trace regression model,
if the medians of the noise variables are zero,
the median \MC{} estimator can be treated as a robust
estimation of $\bm{A}_{\star}$.
Although median is one of the most commonly used robust statistics,
the median \MC{} methods have
not been studied until recently.
\citet{Elsener-Geer18} derived the asymptotic behavior of the
trace-norm regularized estimators under both the absolute deviation loss and the Huber loss. Their convergence rates match with the rate obtained in \citet{Koltchinskii-Lounici-Tsybakov11} under certain conditions. More complete asymptotic results have been developed in \citet{Alquier-Cottet-Lecue19}, which derives the minimax rates of convergence with any Lipschitz loss functions including absolute deviation loss.

To the best of our knowledge,
the only existing computational algorithm of median \MC{} in the literature is
proposed by \citet{Alquier-Cottet-Lecue19},
which is an alternating direction method of multiplier (ADMM) algorithm developed for the
quantile \MC{} with median \MC{} being a special case.
However, this algorithm is slow and not scalable to large matrices
due to the non-smooth nature of both the absolute deviation loss and the regularization term.

Despite the computational challenges,
the absolute deviation loss has a few appealing properties as
compared to the Huber loss.
First, absolute deviation loss is tuning-free while Huber loss has a tuning parameter,
which is equivalent to the tuning parameter in the entry-wise $\ell_1$ penalty
in the low-rank-plus-sparse model.
Second, absolute deviation loss is generally more robust than Huber loss.
Third, the minimizer of expected absolute deviation loss is naturally tied to median, and is generally
more interpretable than the minimizer of expected Huber loss (which may vary with its tuning parameter).

\subsection{Our Goal and Contributions}

Our goal is to
develop a robust and scalable estimator
for median \MC{}, in large-scale problems.
The proposed estimator approximately solves
the regularized ERM with the non-differentiable absolute deviation loss.
It is obtained through two major stages ---
(1) a fast and simple initial estimation via embarrassingly parallel computing
and (2) a refinement stage based on pseudo data.
As pointed out earlier (with more details in Section \ref{sec:BLADMC}), existing computational strategy \citep{Alquier-Cottet-Lecue19}
does not scale well with the dimensions of the matrix.
Inspired by \citet{Mackey-Talwalkar-Jordan15},
a simple strategy is to divide the target matrix into small sub-matrices and
perform median \MC{}
on every sub-matrices in an \textit{embarrassingly} parallel fashion,
and then naively concatenate all estimates of these sub-matrices to form an initial estimate of the target matrix.
Therefore, most computations are done on
much smaller sub-matrices, and hence this computational strategy is much more scalable.
However,
since low-rankness is generally a global (whole-matrix) structure,
the lack of communications between the computations of different
sub-matrices
lead to sub-optimal estimation \citep{Mackey-Talwalkar-Jordan15}.
The key innovation of this paper is a fast refinement stage,
which transforms the regularized ERM with absolute deviation loss into a
regularized ERM with quadratic loss, for which many fast algorithms
exist, via the idea of \textit{pseudo data}.
Motivated by \citet{Chen-Liu-Mao19}, we
develop the pseudo data based on
a Newton-Raphson iteration in expectation.
The construction of the pseudo data requires only a rough initial estimate (see
Condition (C6) in Section \ref{sec:theory}), which is obtained in the first stage. As compared to Huber-loss-based methods (sparse-plus-low-rank
model), the underlying absolute deviation loss is non-differentiable, leading
to computational difficulty for large-scale problems.
The proposed strategy involves a novel refinement stage
to efficiently combine and improve the embarrassingly parallel sub-matrix estimations.

We are able to theoretically show that this refinement stage
can improve the convergence rate of the sub-optimal initial estimator
to near-optimal order,
as good as the computationally expensive median \MC{} estimator of \citet{Alquier-Cottet-Lecue19}. 
To the best of our knowledge, this theoretical guarantee for distributed
computing
is the first of its kind in the literature of matrix completion.

\section{Model and Algorithms}
\subsection{Regularized Least Absolute Deviation Estimator}\label{sec:lad}
Let $\bA_{\star}=(A_{\star,ij})_{i,j=1}^{n_1,n_2}\in \mbR^{n_1\times n_2}$ be an
unknown high-dimensional matrix. Assume the $N$ pairs of observations
$\{(\bX_k,Y_k)\}_{k=1}^N$ satisfy the trace regression model
\eqref{eqn:tracereg}
with noise $\{\varepsilon_k\}_{k=1}^N$.
The design matrices are assumed to be i.i.d.~random matrices
that take values in $\mathcal{X}$ \eqref{eqn:canonical-basis}.
Let $\pi_{st}=\Pr(\bX_k=\be_s(n_1)\be_t^{\rm T}(n_2))$ be the probability
of observing (a noisy realization of) the $(s,t)$-th entry of $\bA_{\star}$ and denote
$\bm{\Pi}=(\pi_{1,1},\dots,\pi_{n_1,n_2})^{\rm T}$.
Instead of the uniform sampling where $\pi_{st}\equiv \pi$
\citep{Koltchinskii-Lounici-Tsybakov11,Rohde-Tsybakov11,Elsener-Geer18},
out setup allows sampling probabilities to be different across entries, such as in
\citet{Klopp14,Lafond15,Cai-Zhou16,Alquier-Cottet-Lecue19}. See Condition (C1)
for more details.
Overall, $(Y_1, \bX_1, \varepsilon_1), \dots, (Y_N, \bX_N, \varepsilon_N)$ are i.i.d. tuples of random variables.
For notation's simplicity,
we let $(Y, \bX, \varepsilon)$ be a generic independent tuple of
random variables that have the same distribution as $(Y_1,\bX_1,\varepsilon_1)$.
Without additional specification, the noise variable $\varepsilon$ is not identifiable.
For example, one can subtract a constant from all entries of $\bA_\star$
and add this constant to the noise.
To identify the noise, we assume
$\mbP(\epsilon\le 0)=0.5$,
which naturally leads to an interpretation of
$\bA_{\star}$ as median, i.e., $A_{\star,ij}$ is the median of $Y\mid \bX=\be_i(n_1)\be_j(n_2)^{\rm T}$.
If the noise distribution is symmetric
and light-tailed (so that the expectation exists),
then $\mbE(\varepsilon_k)=0$, and
$\bA_{\star}$ is the also the mean matrix ($A_{\star,ij}=\mbE(Y\mid \bX=\be_i(n_1)\be_j(n_2)^{\rm T})$),
which aligns with the target of common \MC{} techniques \citep{Elsener-Geer18}.
Let $f$ be the probability density function of the noise.
For the proposed method, the required condition of $f$ is specified in Condition (C3) of Section \ref{sec:theory},
which is fairly mild and is satisfied by many heavy-tailed distributions
whose expectation may not exist.

Define a hypothesis class $\mathcal{B}(a,n,m)=\{\bA\in\mbR^{n\times
	m}:\norm{\bA}_{\infty}\le a\}$
where $a>0$ such that $\bA_{\star}\in \mathcal{B}(a, n, m)$.
In this paper, we use the absolute deviation loss instead of the common
quadratic loss
\citep[e.g.,][]{Candes-Plan10,Koltchinskii-Lounici-Tsybakov11,Klopp14}.
According to Section 4 of the Supplementary Material \citep{Elsener-Geer18},
$\bA_\star$ is also characterized as the minimizer of the
population risk:
\begin{equation}\label{eqn:sto_opt}
\bA_{\star}=\underset{\bA\in\mathcal{B}(a,n_1,n_2)}{\arg\min}\mbE\left\{\Abs{Y-\tr(\bX^{\rm T}\bA)}\right\}.
\end{equation}
To encourage a low-rank solution, one natural candidate is
the following regularized empirical risk estimator \citep{Elsener-Geer18,Alquier-Cottet-Lecue19}:
\begin{align}\label{eqn:candt}
\wh\bA_{\text{LADMC}}=\underset{\bA\in\mathcal{B}(a,n_1,n_2)}{\arg\min}&\frac{1}{N}\sum_{k=1}^{N}\Abs{Y_{k}-\tr(\bX_k^{\rm T}\bA)}\nonumber\\
&+\lambda_{N}^{\prime}\Norm{\bA}_{\ast},
\end{align}
where $\norm{\bA}_{\ast}$ denotes the nuclear norm and $\lambda_{N}^{\prime}\ge 0$ is
a tuning parameter. The nuclear norm is a convex relaxation of the rank which
flavors the optimization and analysis of the statistical property \citep{Candes-Recht09}.

Due to non-differentiability of the absolute deviation loss, the objective function in \eqref{eqn:sto_opt} is the sum of two non-differentiable
terms, rendering common computational strategies based on
proximal gradient method \citep[e.g.,][]{Mazumder-Hastie-Tibshirani10, Wong-Lee17} inapplicable.
To the best of our knowledge,
there is only one existing computational algorithm for \eqref{eqn:sto_opt},
which is based on a direct application of alternating direction method of multiplier (ADMM) \citep{Alquier-Cottet-Lecue19}.
However, this algorithm is slow and not scalable in practice, when the
sample size and the matrix dimensions are large,
possibly due to the non-differentiable nature of the loss.

We aim to derive a computationally efficient method
for estimating the median matrix $\bA_\star$ in large-scale \MC{} problems.
More specifically, the proposed method
consists of two stages:
(1) an initial estimation via distributed computing (Section \ref{sec:BLADMC})
and (2) a refinement stage to achieve near-optimal estimation (Section \ref{sec:refine}).

\subsection{Distributed Initial Estimator}\label{sec:BLADMC}

\begin{figure}
	\centering
	\includegraphics[width=7cm]{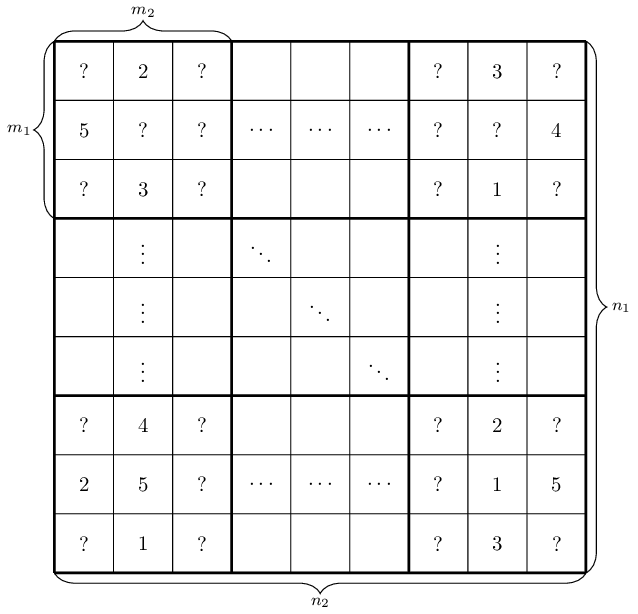}
	\caption{An example of dividing a matrix into sub-matrices.}
	\label{fig:submatrices}
\end{figure}

Similar to many large-scale problems,
it is common to harness distributed computing
to overcome computational barriers.
Motivated by \citet{Mackey-Talwalkar-Jordan15},
we divide the underlying matrix into several sub-matrices,
estimate each sub-matrix separately in an embarrassingly parallel fashion
and then combine them to form a computationally efficient (initial) estimator of $\bA_{\star}$.

For the convenience of notations, suppose there exist integers $m_1$, $m_2$, $l_1$ and
$l_2$ such that $l_1=n_1/m_1$ and $l_2=n_2/m_2$.
(Otherwise, the following description can be easily extended
with
$l_1=\lfloor n_1/m_1\rfloor$ and $l_2=\lfloor n_2/m_2 \rfloor$
which leads to slightly different sizes in several sub-matrices.)
We divide the row indices $1,\dots, n_1$ into $l_1$ subsets evenly where each subset contains $m_1$ index and similarly divide the column indices $1,\dots, n_2$ into $l_2$ subsets evenly.
Then we obtain $l_1l_2$ sub-matrices, denoted by
$\bA_{\star,l}\in \mathbb{R}^{m_1\times m_2}$ for $l=1,\dots, l_1l_2$.
See Figure \ref{fig:submatrices} for a pictorial illustration.
Let $\bm{\Omega}_l$ be the index set of the observed entries within the $l$-th
sub-matrix $\bA_{\star,l}$, and $N_l$ be the corresponding number of observed entries.
Next, we apply the ADMM algorithm of \citet{Alquier-Cottet-Lecue19}
to each sub-matrix $\bA_{\star,l}$ and obtain
corresponding median estimator:
\begin{align}\label{eqn:intial}
\wh\bA_{\text{LADMC},l}=\underset{\bA_{l}\in\mathcal{B}(a,m_1,m_2)}{\arg\min}&\frac{1}{N_{l}}\sum_{k\in\bm{\Omega}_{l}}\Abs{Y_{k}- \tr(\bX_{l,k}^{\rm T}\bA_{l})}\nonumber\\
&+\lambda_{N_{l},l}\Norm{\bA_l}_{\ast},
\end{align}
where $\bX_{l,k}$ is a corresponding sub-design matrix of dimensions $m_1\times m_2$ and
$\lambda_{N_{l},l}\ge 0$ is a tuning parameter.
Note that the most computationally intensive sub-routine in the ADMM algorithm of \citet{Alquier-Cottet-Lecue19} is (repeated applications of) SVD.
For sub-matrices of dimension $m_1\times m_2$, the computational complexity of a single SVD reduced from $\mathcal{O}(n_1^2n_2+n_1n_2^2)$ to $\mathcal{O}(m_1^2m_2+m_1m_2^2)$.

After we have all the $\wh\bA_{\text{LADMC},l}$ for $l=1,\dots,l_1l_2$, we can put these estimators of the sub-matrices back together according to their original positions in the target matrix (see Figure \ref{fig:submatrices}),
and form an initial estimator $\wh \bA_{\text{LADMC},0}$.

This computational strategy is conceptually simple and easily implementable.
However, despite the low-rank estimations for each sub-matrix, combining them directly cannot guarantee low-rankness of $\wh \bA_{\text{LADMC},0}$.
Also, the convergence rate of $\wh \bA_{\text{LADMC},0}$
is not guaranteed to be (near-)optimal,
as long as $m_1, m_2$ are of smaller order than $n_1, n_2$ respectively.
See Theorem \ref{thm:init}(i) in
Section \ref{sec:theory}.
However, for computational benefits,
it is desirable to choose small $m_1, m_2$.
In the next section, we leverage this initial estimator
and formulate a refinement stage.

\subsection{The Idea of Refinement}
\label{sec:refine}

The proposed refinement stage is based on
a form of pseudo data, which
leverages the idea from the Newton-Raphson iteration.
To describe this idea, we start from the stochastic optimization problem
\eqref{eqn:sto_opt}. 
Write the loss function as $L(\bA;\{Y,\bX\}) = \abs{Y- \tr(\bX^{\rm T}\bA)}$.
To solve this stochastic optimization problem, the population version of the Newton-Raphson iteration takes the following form
\begin{align}\label{eq:onestep}
\text{vec}(\bA_{1})=\text{vec}(\wh\bA_{0})
-\bH(\wh \bA_{0})^{-1}\mbE_{(Y,\bX)}\left[\bm{l}(\wh\bA_{0};\{Y,\bX\})\right],
\end{align}
where $(Y, \bX)$ is defined in Section \ref{sec:lad} to be independent of the data; $\text{vec}(\bA)$ is the vectorization of the matrix $\bA$; $\wh
\bA_{0}$ is an initial estimator (to be specified below); and $\bm{l}(\bA;\{Y,\bX\})$ is the sub-gradient of $L(\bA;\{Y,\bX\})$ with respect to $\text{vec}(\bA)$. 
One can show that the population Hessian matrix
takes the form $\bH(\bA)=2\mbE_{(Y,\bX)}(f\{\tr(\bX^{\rm T}(\bA-\bA_{\star})\})\diag(\bm{\Pi})$,
where we recall that $\bm{\Pi}=(\pi_{1,1},\dots,\pi_{n_1,n_2})^{\rm T}$
is the vector of observation probabilities; and
$\diag(\cdot)$ transforms a vector into a diagonal matrix
whose diagonal is the vector.
Also, it can be shown that
$\mbE_{(Y,\bX)}[\bm{l}(\bA;\{Y,\bX\})]=\bm{\Pi}\mbE_{(Y,\bX)}\{2\Ind{Y-\tr(\bX^{\rm T}\bA)\leq 0}-1\}$. 
Recall that $f(x)$ is the density function of
the noise $\epsilon$.

By using $\bH(\bA_{\star})= 2f(0)\diag(\bm{\Pi})$ in
\eqref{eq:onestep},
we obtain the following approximation.
When the initial estimator $\wh \bA_{0}$ is close to the minimizer
$\bA_{\star}$,
\begingroup
\allowdisplaybreaks
\begin{align}\label{eqn:motivate}
&\text{vec}(\bA_{1}) \nonumber
\approx \text{vec}(\wh \bA_{0})\nonumber\\
&\qquad -[2 f(0)\diag(\bm{\Pi})]^{-1}\mbE_{(Y,\bX)} [\bm{l}(\wh  \bA_{0};\{Y,\bX\})] \nonumber\\
&=\mbE_{(Y,\bX)}\left\{\text{vec}(\wh \bA_{0})\right. \nonumber\\
&\qquad -\left.[f(0)]^{-1}\left(\Ind{Y\leq\tr(\bX^{\rm T}\wh\bA_{0})}-\frac{1}{2}\right)\mathds{1}_{n_1n_2}\right\}\nonumber\\
&=[\diag(\bm{\Pi})]^{-1}\mbE_{(Y,\bX)}\left[\text{vec}(\bX)\left\{\text{vec}(\bX)^{\rm T}\text{vec}(\wh \bA_{0})\right.\right.\nonumber\\
&\qquad\left.\left.-[f(0)]^{-1}\left(\Ind{Y\leq\tr(\bX^{\rm T}\wh\bA_{0})}-\frac{1}{2}\right)\right\}\right]\nonumber\\
&=[\diag(\bm\Pi)]^{-1}\mbE_{(Y,\bX)}\left(\mathrm{vec}(\bX)\tilde{Y}^o\right)\\
&= \{\mbE_{(Y,\bX)}[\text{vec}(\bX)\text{vec}(\bX)^{\rm T}]\}^{-1} \mbE_{(Y,\bX)}\left(\text{vec}(\bX)\tilde{Y}^0\right)\nonumber
\end{align}
\endgroup
where we define the theoretical pseudo data
\[
\wt Y^{o}=\tr(\bX^{\rm T}\wh\bA_{0})-[f(0)]^{-1}\left(\Ind{Y\le \tr(\bX^{\rm T}\wh\bA_{0})}-\frac{1}{2}\right).
\] 
Here $\mathds{1}_{n_1n_2}$ denotes the vector of dimension $n_1n_2$ with all elements equal to 1.
Clearly, \eqref{eqn:motivate} is the vectorization of the solution to
$\arg\min_{\bA}\mbE_{(Y, \bX)}\{\wt Y^{o}-\tr(\bX^{\rm T}\bA)\}^{2}$, where $\tr(\bX^{\rm T}\bA)\} = \text{vec}(\bX)^{\rm T}\text{vec}(\bA)$.
From this heuristic argument, we can approximate the population Newton update
by a least square solution based on the pseudo data $\tilde{Y}^0$,
when we start from an $\wh\bA_0$ close enough to $\bA_\star$.
Without the knowledge of $f(0)$,
the pseudo data cannot be used.
In the above, $f(0)$ can be easily estimated by the kernel density estimator: 
\[
\wh f(0)=\frac{1}{Nh}\sum_{k=1}^{N}K\left(\frac{Y_{k}- \tr(\bX_k^{\rm T}\wh\bA_{0})}{h}\right),
\]
where $K(x)$ is a kernel function which satisfies Condition (C4) and $h>0$ is
the bandwidth.
For each $1\le k\le N$, we define the actual pseudo data $\wt\bY$ used in our proposed procedure to be
\[
\wt Y_{k}=\tr(\bX_k^{\rm T}\wh\bA_{0})-[\wh f(0)]^{-1}\left(\Ind{Y_{k}\le \tr(\bX_k^{\rm T}\wh\bA_{0})}-\frac{1}{2}\right),
\]
and $\wt\bY=(\wt Y_{k})$. 
For finite sample, regularization is imposed to estimate
the high-dimensional parameter $\bA_\star$.
By using $\wt\bY$, one natural candidate for the estimator of $\bA_{\star}$ is given by
\begin{align}\label{eqn:Ahat1}
\wh\bA=\underset{\bA\in\mathcal{B}(a,n_1,n_2)}{\arg\min}&\frac{1}{N}\sum_{k=1}^{N}\left(\wt Y_{k}-\tr(\bX_k^{\rm T}\bA)\right)^2\nonumber\\
&+\lambda_{N}\Norm{\bA}_{\ast},
\end{align}
where $\norm{\cdot}_{\ast}$ is the nuclear norm and $\lambda_{N}\ge 0$ is the tuning
parameter.
If $\wt \bY$ is replaced by $\bY$, the optimization \eqref{eqn:Ahat1}
is a common nuclear-norm regularized empirical risk estimator with quadratic loss
which has been well studied in the literature
\citep{Candes-Recht09,Candes-Plan10,Koltchinskii-Lounici-Tsybakov11,Klopp14}.
Therefore, with the knowledge of $\wt\bY$,
corresponding computational algorithms can be adopted to solve
\eqref{eqn:Ahat1}.
Note that the pseudo data are based on an initial estimator $\wh\bA_0$.
In Section \ref{sec:theory}, we
show that any initial estimator that fulfills Condition (C5)
can be improved by \eqref{eqn:Ahat1}, which is therefore
called a refinement step.
It is easy to verify that
the initial estimator $\wh \bA_{\text{LADMC},0}$
in Section \ref{sec:BLADMC}
fulfills such condition.
Note that the initial estimator, like $\wh\bA_{\text{LADMC},0}$, introduces
complicated dependencies among the entries of $\wt \bY$,
which brings new challenges in analyzing \eqref{eqn:Ahat1},
as opposed to the common estimator based on $\bY$ with independent entries.

From our theory (Section \ref{sec:theory}),
the refined estimator \eqref{eqn:Ahat1} improves
upon the initial estimator.
Depending on how bad the initial estimator is,
a single refinement step may not be good enough to
achieve a (near-)optimal estimator.
But this can remedied by reapplying
the refinement step again and again. In Section \ref{sec:theory},
we show that a finite number of application of the refinement step is
enough. In our numerical experiments, 4--5 applications would usually
produce enough improvement.
Write $\wh\bA^{(1)}=\wh\bA$ given in \eqref{eqn:Ahat1} as the estimator from the first iteration and we can construct an iterative procedure to estimate $\bA_{\star}$. In particular, let $\wh\bA^{(t-1)}$ be the estimator in the $(t-1)$-th iteration. Define 
\[
\wh f^{(t)}(0)=\frac{1}{Nh_{t}}\sum_{k=1}^{N}K\left(\frac{Y_{k}-\tr(\bX_k^{\rm T} \wh\bA^{(t-1)})}{h_{t}}\right),
\]
where $K(x)$ is the same smoothing function used to estimate $f(0)$ in the first step and $h_{t}\to0$ is the bandwidth for the $t$-th iteration. Similarly, for each $1\le k\le N$, define
\begin{align}\label{eqn:pseudoYt}
\wt Y_{k}^{(t)}=&\tr(\bX_k^{\rm T}\wh\bA^{(t-1)})- \left(\wh f^{(t)}\left(0\right)\right)^{-1}\times\nonumber\\
&\left(\Ind{Y_{k}\le \tr(\bX_k^{\rm T}\wh\bA^{(t-1)})}-\frac{1}{2}\right).
\end{align}
We propose the following estimator
\begin{align}\label{eqn:Ahatt}
\wh\bA^{(t)}=\underset{\bA\in\mathcal{B}(a,n_1,n_2)}{\arg\min}&\frac{1}{N}\sum_{k=1}^{N}\left(\wt Y_{k}^{(t)}-\tr(\bX_k^{\rm T}\bA)\right)^2\nonumber\\
&+\lambda_{N,t}\Norm{\bA}_{\ast},
\end{align}
where $\lambda_{N,t}$ is the tunning parameter in the $t$-th iteration.
To summarize, we list the full algorithm in Algorithm \ref{alg:DLADMC}.

\begin{algorithm}[!t]
	\caption{{\small Distributed Least Absolute Deviation Matrix Completion}}
	\label{alg:DLADMC}
	{\textbf{Input:} Observed data pairs $\{\bX_k,Y_k\}$ for $k=1,\ldots, N$, number of observations $N$, dimensions of design matrix $\bX$ $n_1,n_2$, dimensions of sub-matrices to construct the initial estimator $m_1,m_2$ and the split subsets $\bm{\Omega}_{l}$ for $l=1,\dots,l_1l_2$, kernel function $K$, a sequence of bandwidths $h_{t}$ and the regularization parameters $\lambda_{N,t}$ for $t=1,\dots,T$.}
	
	\begin{algorithmic}[1]
		\STATE Get the robust low-rank estimator of each $\bA_{\star,l}$ by solving the minimization problem \eqref{eqn:intial} in parallel.
		\STATE Set $\wh\bA^{(0)}$ to be the same as the initial estimator $\wh\bA_{\text{LADMC},0}$ by putting $\wh\bA_{\text{LADMC},l}$ together.
		\FOR{$t=1,2 \ldots, T$}
		\STATE Compute $\wh f^{(t)}(0):=(Nh_{t})^{-1}\sum_{k=1}^{N}K(h_{t}^{-1}(Y_{k}-\tr\{\bX_{k}^{\rm T}\wh\bA^{(t-1)}\}))$. 
		\STATE Construct the pseudo data $\{\wt Y_{k}^{(t)}\}$ by equation \eqref{eqn:pseudoYt}.
		\STATE Plugin the pseudo data $\{\wt Y_{k}^{(t)}\}$ and compute the estimator $\wh\bA^{(t)}$ by solving the minimization problem \eqref{eqn:Ahatt}.
		\ENDFOR
	\end{algorithmic}
	\textbf{Output:}  The final estimator $\wh\bA^{(T)}$.
\end{algorithm}

\section{Theoretical Guarantee}\label{sec:theory}
To begin with, we introduce several notations.
Let $m_{+}=m_1+m_2$, $m_{\max}=\max\{m_1,m_2\}$ and $m_{\min}=\min\{m_1,m_2\}$.
Similarly, write $n_{+}=n_1+n_2$, $n_{\max}=\max\{n_1,n_2\}$ and
$n_{\min}=\min\{n_1,n_2\}$.
For a given matrix $\bA=(A_{ij})\in\mbR^{n_1\times n_2}$, denote
$\sigma_{i}(\bA)$ be the $i$-th largest singular value of matrix $\bA$. Let
$\norm{\bA}=\sigma_{1}(\bA)$,
$\norm{\bA}_{F}=\sqrt{\sum_{i=1}^{n_1}\sum_{j=1}^{n_2}
	A_{ij}^2}$ and $\norm{\bA}_{\ast}=\sum_{i=1}^{n_{\min}}\sigma_{i}(\bA)$ be
the spectral norm (operator norm), the infinity norm, the Frobenius norm and the trace norm of a matrix $\bA$ respectively. Define a class of matrices
$\mathcal{C}_{\ast}(n_1,n_2)=\{\bA\in\mbR^{n_1\times n_2}:\norm{\bA}_{\ast}\le
1\}$. Denote the rank of matrix $\bA_{\star}$
by $r_{\star}=\text{rank}(\bA_{\star})$
for simplicity.
With these notations, we describe the following conditions
which are useful in our theoretical analysis.

{\bf (C1)} For each $k=1,\dots,N$, the design matrix $\bX_k$ takes value in the canonical basis $\mathcal{X}$ as defined in \eqref{eqn:canonical-basis}. There exist positive constants $\underline{c}$ and $\overline{c}$ such that for any $(s,t)\in\{1,\dots,n_1\}\times\{1,\dots,n_2\}$, $\underline{c}/(n_1n_2)\le \Pr(\bX_k=\be_s(n_1)\be_t^{\rm T}(n_2))\le\overline{c}/(n_1n_2)$.

{\bf (C2)} The local dimensions $m_1,m_2$ on each block satisfies $m_1\ge n_1^c$ and $m_2\ge n_2^c$ for some $0<c<1$. The number of observations in each block $N_l$ are comparable for all $l=1,\dots,l_1l_2$, i.e, $N_l\asymp m_1m_2N/(n_1n_2)$.

{\bf (C3)} The density function $f(\cdot)$ is Lipschitz continuous (i.e.,
$\abs{f(x)-f(y)}\le C_L\abs{x-y}$ for any $x,y\in\mbR$ and some constant
$C_L>0$). Moreover,
there exists a constant $c>0$ such that
$f(u)\ge c$ for any $\abs{u}\le2a$.
Also, $\Pr(\epsilon_{k} \leq 0)=0.5$ for each $k=1,\dots,N$.

\begin{thm}[\citet{Alquier-Cottet-Lecue19}, Theorem 4.6, Initial estimator]\label{thm:init}
	Suppose that Conditions (C1)--(C3) hold and
	$\bA_{\star}\in\mathcal{B}(a,n_1,n_2)$. For each $l=1,\dots,n_1n_2/(m_1m_2)$,
	assume that there exists a matrix with rank
	at most $s_l$ in $\bA_{\star,l}+(\rho_{s_l}/20)\mathcal{C}_{\ast}(m_1,m_2)$
	where
	$\rho_{s_l}=C_{\rho}(s_lm_1m_2)(\log(m_{+})/(m_{+}N_l))^{1/2}$
	with the universal constant $C_{\rho}$.
	
	(i) Then there exist universal constants $c(\underline{c},\overline{c})$ and $C$ such that with $\lambda_{N_l,l}=c(\underline{c},\overline{c})\sqrt{\log(m_{+})/(m_{\min}N_l)}$, the estimator $\wh\bA_{\text{LADMC},l}$ in \eqref{eqn:intial} satisfies

		\begin{align}\label{eqn:QMCupper}
		\frac{1}{\sqrt{m_1m_2}}\Norm{\wh\bA_{\text{LADMC},l}-\bA_{\star,l}}_{F}\le C\min\left\{\sqrt{\frac{s_{l}m_{\max}\log(m_{+})}{N_{l}}},\Norm{\bA_{\star,l}}_{\ast}^{1/2}\left(\frac{\log(m_{+})}{m_{\min}N_{l}}\right)^{1/4}\right\},
		\end{align}

	with probability at least $1-C\exp(-C s_lm_{\max}\log(m_{+}))$.
	
	(ii) Moreover, by putting these $l_1l_2$ estimators $\wh\bA_{\text{LADMC},l}$ together, for the same constant $C$ in (i), we have the initial estimator $\wh\bA_{\text{LADMC},0}$ satisfies

		\begin{align*}
		\frac{\Norm{\wh \bA_{\text{LADMC},0}-\bA_{\star}}_{F}}{\sqrt{n_1n_2}}\le C\min\left\{\sqrt{\frac{\{\sum_{l=1}^{l_1l_2} s_{l}\}m_{\max}\log(m_{+})}{N}},\right.\\
		\left.\left(\sum_{l=1}^{l_1l_2}\Norm{\bA_{\star,l}}_{\ast}^{1/2}\right)\left(\frac{m_{\max}\log(m_{+})}{n_1n_2N}\right)^{1/4}\right\},
		\end{align*}
	
	with probability at least $1-C \exp(\log(n_1n_2)- C m_{\max}\log(m_{+}))$.
\end{thm}

From Theorem \ref{thm:init},
we can guarantee the convergence of the sub-matrix estimator
$\wh\bA_{\text{LADMC},l}$ when $m_1,m_2\to\infty$.
For the initial estimator $\wh \bA_{\text{LADMC},0}$, under Condition (C3) and
that all the sub-matrices are low-rank ($s_l\asymp1$ for all $l$), we require the number of observation $N\ge C_1(m_1m_2)^{-1}(n_1n_2)m_{\max}\log(m_{+})$ for some constant $C_1$ to ensure the convergence.
As for the rate of convergence, $\sqrt{(n_1n_2)m_{\max}\log(m_{+})/(Nm_1m_2)}$ is slower than the classical optimal rate $\sqrt{r_{\star}n_{\max}\log(n_{+})/N}$
when $m_1, m_2$ are of smaller than $n_1, n_2$ respectively.

{\bf (C4)} Assume the kernel functions $K(\cdot)$ is integrable with $\int_{-\infty}^{\infty}K(u)du=1$. Moreover, assume that $K(\cdot)$ satisfies $K(u)=0$ if $\abs{u}\ge 1$. Further, assume that $K(\cdot)$ is differentiable and its derivative $K^{\prime}(\cdot)$ is bounded.

{\bf (C5)} The initial estimator $\wh\bA_{0}$ satisfies $(n_1n_2)^{-1/2}\norm{\wh\bA_{0}-\bA_{\star}}_{F}=O_{\pr}((n_1n_2)^{-1/2}a_{N})$, where the initial rate $(n_1n_2)^{-1/2}a_{N}=o(1)$.

For the notation consistency, denote the initial rate $a_{N,0}=a_{N}$ and define that
\begin{align}\label{eqn:aNt}
a_{N,t}=\sqrt{\frac{r_{\star}(n_1n_2)n_{\max}\log (n_{+})}{N}}+\frac{n_{\min}}{\sqrt{r_{\star}}}\left(\frac{\sqrt{r_{\star}}a_{N,0}}{n_{\min}}\right)^{2^t}
\end{align}

\begin{thm}[Repeated refinement]\label{thm:keyt}
	Suppose that Conditions (C1)--(C5) hold and $\bA_{\star}\in\mathcal{B}(a,n_1,n_2)$. By choosing the bandwidth $h_{t}\asymp (n_1n_2)^{-1/2}a_{N,t-1}$ where $a_{N,t}$ is defined as in \eqref{eqn:aNt} and taking 
	\begin{align*}
	\lambda_{N,t}=C\left(\sqrt{\frac{\log (n_{+})}{n_{\min}N}}+\frac{a_{N,t-1}^{2}}{n_{\min}(n_1n_2)}\right),
	\end{align*}
	where $C$ is a sufficient large constant, we have
	\begin{align}\label{eqn:rmset}
	\frac{\Norm{\wh \bA^{(t)}-\bA_{\star}}_{F}^2}{n_1n_2}=
	O_{\pr}\left[\max\left\{\sqrt{\frac{\log(n_{+})}{N}},r_{\star}\left(\frac{n_{\max}\log (n_{+})}{N}+\frac{a_{N,t-1}^{4}}{n_{\min}^2(n_1n_2)}\right)\right\}\right].
	\end{align}
\end{thm}

When the iteration number $t=1$, it means one-step refinement from the initial estimator $\wh\bA_{0}$. For the right hand side of \eqref{eqn:rmset}, it is noted that both the first term $\sqrt{\log(n_{+})/N}$ and the second term $r_{\star}n_{\max}\log(n_{+})/N$ are
seen in the error bound of existing works \citep{Elsener-Geer18,Alquier-Cottet-Lecue19}.
The bound has an extra third term $r_{\star}a_{N,0}^4/(n_{\min}^2(n_1n_2))$ due to
the initial estimator.
After one round of refinement, one can see that the third term $r_{\star}a_{N,0}^4/(n_{\min}^2(n_1n_2))$ in \eqref{eqn:rmset}
is faster than
$a_{N,0}^2/(n_1n_2)$, the convergence rate of the initial estimator (see Condition (C5)),
because $r_{\star}n_{\min}^{-2}a_{N,0}^2=o(1)$.

With the increasing of the iteration number $t$, Theorem \ref{thm:keyt} shows that the estimator can be refined
again and again, until near-optimal rate of convergence is achieved. It can be shown that when the iteration number $t$ exceeds certain number, i.e,
\[
t\ge\log\left\{\frac{\log(r_{\star}^2n_{\max}^2\log(n_{+}))-\log(n_{\min}N)}{c_{0}\log(r_{\star}a_{N,0}^2)-2c_{0}\log(n_{min})}\right\}/\log(2),
\]
for some $c_{0}>0$, the second term in the term associated with $r_{\star}$ is dominated by the first term and the convergence rate of $\wh \bA^{(t)}$ becomes $r_{\star}n_{\max}N^{-1}\log(n_{+})$ which is the near-optimal rate $r_{\star}n_{\max}N^{-1}$ (optimal up to a logarithmic factor). Note that the number of iteration $t$ is usually small due to the logarithmic transformation. 

\subsection{Main Lemma and Proof Outline}
For the $t$-th refinements, let $\xi_{k}^{(t)}=\wt
Y_{k}^{(t)}-\inner{\bX_k}{\bA_{\star}}$ be the residual of the pseudo data.
Also, define the stochastic terms
$\bm{\Sigma}^{(t)}=N^{-1}\sum_{k=1}^{N}\xi_k^{(t)}\bX_k$.
To provide an upper bound of $(n_1n_2)^{-1}\norm{\wh \bA^{(t)}-\bA_{\star}}_{F}^2$ in Theorem \ref{thm:keyt}, we follow the standard arguments, as used in corresponding key theorems in, e.g., \citet{Koltchinskii-Lounici-Tsybakov11,Klopp14}.
The key is to control the spectral norm of the stochastic term $\bm{\Sigma}^{(t)}$.
A specific challenge of our setup  is 
the dependency among the residuals $\{\xi_i^{(t)}\}$.
We tackle this by the following lemma:

\begin{lem}\label{lem:errort}
	Suppose that Conditions (C1)--(C5) hold and $\bA_{\star}\in\mathcal{B}(a,n_1,n_2)$. For any iteration $t\ge 1$, we choose the bandwidth $h_{t}\asymp (n_1n_2)^{-1/2}a_{N,t}$ where $a_{N,t}$ is defined as in \eqref{eqn:aNt}. Then we have
	\begin{align*}
	\Norm{\bm{\Sigma}^{(t)}}
	=&O_{\pr}\left(\sqrt{\frac{\log (n_{+})}{n_{\min}N}}+\frac{a_{N,t-1}^{2}}{n_{\min}(n_1n_2)}\right).
	\end{align*}
\end{lem}

We now give a proof outline of Lemma \ref{lem:errort} for $t=1$.
The same argument can be applied iteratively to achieve the
repeated refinement results as shown in Lemma \ref{lem:errort}.

In our proof, we decompose the stochastic term $\Sigma^{(1)}$ into three components $\bH_{N}(\wh\bA_0)$,\\
$(N\wh f(0))^{-1}\sum_{i=1}^{N}[\bX_{i}\Ind{\epsilon_{i}\le 0}-\bX_{i}f(0)]$
 and $\wh f^{-1}(0)U_{N}$
  where
\begin{align*}\label{eqn:Hn}
\bH_{N}(\bA)=\frac{1}{N}\sum_{i=1}^{N}\bX_{i}\tr\left\{\bX_{i}^{\rm T}\left(\bA-\bA_{\star}\right)\right\}+
\frac{\wh f^{-1}\left(0\right)}{N}\sum_{i=1}^{N}\bX_{i}\left\{f\left[ \tr\left\{\bX_{i}^{\rm T}\left(\bA-\bA_{\star}\right)\right\}\right]-f\left(0\right)\right\},
\end{align*}
and $U_{N}=\underset{\Norm{\bA-\bA_{\star}}_{F}\le a_{N}}{\sup}\Norm{\bB_{N}(\bA)}$
with
\begin{eqnarray*}\label{eqn:Bn}
	\bB_{N}(\bA)=\frac{1}{N}\sum_{i=1}^{N}\left[\bX_{i}\Ind{\epsilon_{i}\le \tr\left\{\bX_{i}^{\rm T}\left(\bA-\bA_{\star}\right)\right\}}\right.
	\left.-\bX_{i}f\left( \tr\left\{\bX_{i}^{\rm T}\left(\bA-\bA_{\star}\right)\right\}\right)\right]\cr
	-\frac{1}{N}\sum_{i=1}^{N}\left[\bX_{i}\Ind{\epsilon_{i}\le 0}-\bX_{i}f\left(0\right)\right].
\end{eqnarray*}
Then we control their spectral norms separately.

For $\bH_{N}(\wh\bA_0)$, we first bound $\abs{\bv^{\rm T}\bH_{N}(\wh\bA_0)\bu}$ for fixed $\bu$ and $\bv$ where $\|\bu\|=\|\bv\|=1$, by separating the random variables $\bX_k$ and $\epsilon_k$ from $\wh\bA_0-\bA_{\star}$, and then applying the exponential inequality in Lemma 1 of \citet{Cai-Liu11}.
To control the spectral norm, 
we take supremum over $\bu$ and $\bv$, and the corresponding uniform bound can be derived using an $\mathcal{E}-$net argument.
The same technique can be used to handle the term $U_N$.
Therefore,  we can bound the spectral norm of $\bH_{N}(\wh\bA_0)$ and $U_{N}$
for any initial estimator that satisfies Condition (C5).

As for the term $(N\wh f(0))^{-1}\sum_{i=1}^{N}[\bX_{i}\Ind{\epsilon_{i}\le 0}-\bX_{i}f(0)]$,
we first note that it is not difficult to control a simplified version: $(Nf(0))^{-1}\sum_{i=1}^{N}[\bX_{i}\Ind{\epsilon_{i}\le 0}-\bX_{i}f(0)]$,
with $f(0)$ instead of $\hat{f}(0)$.
To control our target term,
we provide Proposition \ref{appprop:f0} in the supplementary materials which shows that $\abs{\wh f(0)-f(0)}=O_{\pr}(\sqrt{\frac{\log (n_{+})}{Nh}}+\frac{a_{N}}{\sqrt{n_1n_2}})$.

\section{Experiments}
\subsection{Synthetic Data}\label{sec:synthetic}
We conducted a simulation study, under which we fixed the dimensions to $n_1=n_2=400$.
In each simulated data, the target matrix $\bA_\star$
was generated as $\bU\bV^{\rm T}$, where
the entries of $\bU\in\mbR^{n_1\times r}$ and $\bV\in\mbR^{n_2\times r}$ were all drawn from the standard normal distributions $\mathcal{N}(0,1)$ independently.
Here $r$ was set to $3$.
Thus $\bA_{\star}=\bU\bV^{\rm T}$ was a low-rank matrix.
The missing rate was $0.2$, which corresponds to $N=32,000$
We adopted the uniform missing mechanism where all entries had the same chance of being observed.
We considered the following four noise distributions:
\begin{itemize}
	\item[S1] Normal: $\epsilon\sim\mathcal{N}(0,1)$.
	\item[S2] Cauchy:  $\epsilon\sim\text{Cauchy}(0,1)$.
	\item[S3] Exponential: $\epsilon\sim\text{exp}(1)$.
	\item[S4] t-distribution with degree of freedom $1$: $\epsilon\sim\text{t}_1$.
\end{itemize}
We note that Cauchy distribution is a very heavy-tailed distribution and its first moment (expectation) does not exist. 
For each of these four settings, we repeated the simulation for 500 times.

Denote the proposed median \MC{} procedure given in Algorithm \ref{alg:DLADMC}
by $\textsf{DLADMC}$ (Distributed Least Absolute Deviations Matrix Completion).
Due to Theorem \ref{thm:init}(ii), $\norm{\wh \bA_{\text{LADMC},0}-\bA^{\star}}_{F}=O_{p}(\sqrt{(n_1n_2)^2m_{\max}\log(m_{+})/(m_1m_2N)})$,
we fixed
\[
a_{N}=a_{N,0}=c_1\sqrt{\frac{(n_1n_2)^2m_{\max}\log(m_{+})}{m_1m_2N}},
\]
where the constant $c_1=0.1$. From our experiences, smaller $c_1$ leads to similar results. As $h\asymp (n_1n_2)^{-1/2}a_{N}$, the bandwidth $h$ was simply set to $h=c_2(n_1n_2)^{-1/2}a_{N}$, and similarly, $h_{t}=c_2(n_1n_2)^{-1/2}a_{N,t}$ where $a_{N,t}$ was defined by \eqref{eqn:aNt} with $c_2=0.1$. In addition, all the tuning parameters $\lambda_{N,t}$ in Algorithm \ref{alg:DLADMC} were chosen by validation. Namely, we minimized the absolute deviation loss evaluated on an independently generated validation sets with the same dimensions $n_1,n_2$.
For the choice of the kernel functions $K(\cdot)$, we adopt the commonly used bi-weight kernel function,
\[
\small
K(x) = \begin{cases}
0, & x\leq -1\\
-\frac{315}{64}x^6+\frac{735}{64}x^4-\frac{525}{64}x^2+\frac{105}{64}, & -1\leq x\leq 1\\
0, & x \geq 1
\end{cases}.
\] 
It is easy to verify that $K(\cdot)$ satisfies Condition (C1) in Section \ref{sec:theory}.
If we compute $e=\norm{\wh\bA^{(t)}-\wh\bA^{(t-1)}}_{F}^2/\norm{\wh\bA^{(t-1)}}_{F}^2$ and stop the algorithm once $e\le 10^{-5}$,
it typically only requires $4-5$ iterations. We simply report the results of the estimators with $T=4$ or $T=5$ iterations in Algorithm \ref{alg:DLADMC} (depending on the noise distribution).

We compared the performance of the proposed method ($\textsf{DLADMC}$) with three other approaches:
\begin{itemize}
	\item[(a)] $\textsf{BLADMC}$: Blocked Least Absolute Deviation Matrix Completion $\wh \bA_{\text{LADMC},0}$, the initial estimator proposed in section \ref{sec:BLADMC}. Number of row subsets $l_1=2$, number of column subsets $l_2=2$. 
	\item[(b)] $\textsf{ACL}$: Least Absolute Deviation Matrix Completion with nuclear norm penalty based on the computationally expensive ADMM algorithm proposed by \citet{Alquier-Cottet-Lecue19}.
	\item[c)] $\textsf{MHT}$: The squared loss estimator with nuclear norm penalty proposed by \citet{Mazumder-Hastie-Tibshirani10}.
\end{itemize}
The tuning parameters in these four methods were chosen based on the same validation set.
We followed the selection procedure in Section 9.4 of \citet{Mazumder-Hastie-Tibshirani10} to choose $\lambda$. Instead of fixing $K$ to 1.5 or 2 as in \citet{Mazumder-Hastie-Tibshirani10}, we choose $K$ by an additional pair of training and validation sets (aside from the 500 simulated datasets). We did this for every method to ensure a fair comparison.
The performance of all the methods were evaluated via root mean square error (RMSE) and mean absolute error (MAE). The estimated ranks are also reported. 

\begin{table}[h]
	\centering
	\footnotesize
	\caption{The average RMSEs, MAEs, estimated ranks and their standard
		errors (in parentheses) of \textsf{DLADMC}, \textsf{BLADMC},  \textsf{ACL} and \textsf{MHT} over 500 simulations. The number in the first column within the parentheses represents $T$ in Algorithm \ref{alg:DLADMC} for \textsf{DLADMC}.}
	\label{tab:sim}
	\begin{tabular}{r|r|r|r}
		\hline
		(T) &  & \textsf{DLADMC} & \textsf{BLADMC} \\
		\hline
		\multirow{4}{*}{S1(4)} & \text{RMSE} & 0.5920 (0.0091) & 0.7660 (0.0086) \\ 
		& \text{MAE} & 0.4273 (0.0063) & 0.5615 (0.006) \\ 
		& \text{rank} & 52.90 (2.51) & 400 (0.00)  \\ 
		\hline
		\multirow{4}{*}{S2(5)} & \text{RMSE} & 0.9395 (0.0544) & 1.7421 (0.3767) \\ 
		& \text{MAE} & 0.6735 (0.0339) & 1.2061 (0.1570)  \\ 
		& \text{rank} & 36.49 (7.94) & 272.25 (111.84) \\ 
		\hline
		\multirow{4}{*}{S3(5)} & \text{RMSE} & 0.4868 (0.0092) & 0.6319 (0.0090)  \\ 
		& \text{MAE} & 0.3418 (0.0058) & 0.4484 (0.0057)  \\ 
		& \text{rank} & 66.66 (1.98) & 400 (0.00) \\ 
		\hline
		\multirow{4}{*}{S4(4)} & \text{RMSE} & 1.1374 (0.8945) & 1.6453 (0.2639) \\ 
		& \text{MAE} & 0.8317 (0.7370) & 1.1708 (0.1307)  \\ 
		& \text{rank} & 47.85 (13.22) & 249.16 (111.25)  \\ 
		\hline
		(T) &  & \textsf{ACL} & \textsf{MHT} \\
		\hline
		\multirow{4}{*}{S1(4)} & \text{RMSE} & 0.5518 (0.0081) & 0.4607 (0.0070) \\ 
		& \text{MAE} & 0.4031 (0.0056) & 0.3375 (0.0047) \\ 
		& \text{rank} & 400 (0.00) & 36.89 (1.79) \\
		\hline
		\multirow{4}{*}{S2(5)} & \text{RMSE} & 1.8236 (1.1486) & 106.3660 (918.5790) \\ 
		& \text{MAE} & 1.2434 (0.5828) & 1.4666 (2.2963) \\ 
		& \text{rank} & 277.08 (170.99) & 1.25 (0.50) \\ 
		\hline
		\multirow{4}{*}{S3(5)} & \text{RMSE} & 0.4164 (0.0074) & 0.4928 (0.0083) \\ 
		& \text{MAE} & 0.3121 (0.0054) & 0.3649 (0.0058) \\ 
		& \text{rank} & 400 (0.00) & 37.91 (1.95) \\ 
		\hline
		\multirow{4}{*}{S4(4)} & \text{RMSE} & 1.4968 (0.6141) & 98.851 (445.4504) \\ 
		& \text{MAE} & 1.0792 (0.3803) & 1.4502 (1.1135) \\ 
		& \text{rank} & 237.05 (182.68) & 1.35 (0.71) \\
		\hline
	\end{tabular}
\end{table}

From Table \ref{tab:sim}, we can see that both \textsf{DLADMC} and \textsf{MHT} produced low-rank estimators while \textsf{BLADMC} and \textsf{ACL} could not reduce the rank too much. As expected, when the noise is Gaussian, \textsf{MHT} performed best in terms of RMSE and MAE. Meanwhile, \textsf{DLADMC} and \textsf{ACL} were close to each other and slightly worse than \textsf{MHT}. It is not surprising that \textsf{BLADMC} was the worst due to its simple way to combine sub-matrices. As for Setting S3, \textsf{ACL} outperformed other three methods while the performances of \textsf{DLADMC} and \textsf{MHT} are close. For the heavy-tailed Settings S2 and S4, our proposed \textsf{DLADMC} performed significantly better than \textsf{ACL}, and \textsf{MHT} fails.

Moreover, to investigate whether the refinement step can be isolated from the distributed optimization, we run the refinement step on an initial matrix that is synthetically generated by making small noises to the ground-truth matrix $\bA_{\star}$, as suggested by a reviewer. We provide these results in Section \ref{appsec:syn} of the supplementary material.

\subsection{Real-World Data}
We tested various methods on the MovieLens-100K\footnote{\url{https://grouplens.org/datasets/movielens/100k/}} dataset. This data set consists of 100,000 movie ratings provided by 943 viewers on 1682 movies. The ratings range from 1 to 5. To evaluate the performance of different methods, we directly used the data splittings from the data provider, which splits the data into two sets. We refer them to as \textsf{RawA} and \textsf{RawB}. Similar to \citet{Alquier-Cottet-Lecue19}, we added artificial outliers by randomly changing $20\%$ of ratings 
that are equal to $5$ in the two sets, \textsf{RawA} and \textsf{RawB}, to $1$ and constructed \textsf{OutA} and \textsf{OutB} respectively. To avoid rows and columns that contain too few observations,
we only keep the rows and columns with at least $20$ ratings.
The resulting target matrix $\bA_\star$ is of dimension $739\times 918$.
Before we applied those four methods as described in Section \ref{sec:synthetic},
the data was preprocessed by a bi-scaling procedure \citep{Mazumder-Hastie-Tibshirani10}. For the proposed \textsf{DLADMC}, we fixed the iteration number to $7$.
It is noted that the relative error stopping criterion (in Section \ref{sec:synthetic}) did not result in a stop within the first 7 iteration, where 7 is just a user-specified upper bound in the implementation. To understand the effect of this bound, we provided additional analysis of this upper bound in Section 2.2 of the supplementary material.
Briefly, our conclusion in the rest of this section is not sensitive to this choice of upper bound.
The tuning parameters for all the methods were chosen by 5-fold cross-validations. The RMSEs, MAEs, estimated ranks and the total computing time (in seconds) are reported in Table \ref{tab:real}. 
For a fair comparison, we recorded the time of each method in the experiment with the selected tuning parameter.

\begin{table}[h]
	\centering
	\footnotesize
	\caption{The RMSEs, MAEs and estimated ranks of \textsf{DLADMC}, \textsf{BLADMC}, \textsf{ACL} and \textsf{MHT} under dimensions $n_1=739$ and $n_2=918$.}\label{tab:real}
	\begin{tabular}{r|r|r|r|r|r}
		\hline
		& & \textsf{DLADMC} & \textsf{BLADMC} & \textsf{ACL} & \textsf{MHT} \\
		\hline
		\multirow{4}{*}{\textsf{RawA}} & \text{RMSE} & 0.9235 & 0.9451 & 0.9258 & 0.9166 \\
		& \text{MAE} & 0.7233 & 0.7416 & 0.7252 & 0.7196 \\
		& \text{rank} & 41 & 530 & 509 & 57 \\
		& $t$ & 254.33 & 65.64 & 393.40 & 30.16 \\
		\hline
		\multirow{4}{*}{\textsf{RawB}} & \text{RMSE} & 0.9352 & 0.9593 & 0.9376 & 0.9304 \\
		& \text{MAE} & 0.7300 & 0.7498 & 0.7323 & 0.7280 \\
		& \text{rank} & 51 & 541 & 521 & 58 \\
		& $t$ & 244.73 & 60.30 & 448.55 & 29.60 \\
		\hline
		\multirow{4}{*}{\textsf{OutA}} & \text{RMSE} & 1.0486 & 1.0813 & 1.0503 & 1.0820 \\
		& \text{MAE} & 0.8568 & 0.8833 & 0.8590 & 0.8971 \\
		& \text{rank} & 38 & 493 & 410 & 3 \\
		& $t$ & 255.25 & 89.65 & 426.78 & 10.41 \\
		\hline
		\multirow{4}{*}{\textsf{OutB}} & \text{RMSE} & 1.0521 & 1.0871 & 1.0539 & 1.0862 \\
		& \text{MAE} & 0.8616 & 0.8905 & 0.8628 & 0.9021 \\
		& \text{rank} & 28 & 486 & 374 & 6 \\
		& $t$ & 260.79 & 104.97 & 809.26 & 10.22 \\
		\hline
	\end{tabular}
\end{table}

It is noted that under the raw data \textsf{RawA} and \textsf{RawB}, both the proposed \textsf{DLADMC} and the least absolute deviation estimator \textsf{ACL} performed similarly as the least squares estimator \textsf{MHT}. \textsf{BLADMC} lost some efficiency due to the embarrassingly parallel computing. For the dataset with outliers, the proposed \textsf{DLADMC} and the least absolute deviation estimator \textsf{ACL} performed better than \textsf{MHT}. Although \textsf{DLADMC} and \textsf{ACL} had similar performance in terms of the RMSEs and MAEs, \textsf{DLADMC} required much lower computing cost.

Suggested by a reviewer, we also performed an experiment with a bigger data set (MovieLens-1M dataset: 1,000,209 ratings of approximately 3,900 movies rated by 6,040 users.)
However, ACL is not scalable,
and,  due to time limitations,
we stopped the fitting of ACL when 
the running time of ACL exceeds five times of the proposed DLADMC.
In our analysis, we only compared the remaining methods.
The conclusions were similar as in the smaller MoviewLens-100K dataset.
The details are presented in Section \ref{appsec:real} of the supplementary material.

\section{Conclusion}
In this paper, we address the problem of median \MC{} and obtain a computationally efficient estimator for large-scale \MC{} problems. We construct the initial estimator in an embarrassing parallel fashion and refine it through regularized least square minimizations based on pseudo data.
The corresponding non-standard asymptotic analysis are established. This shows that
the proposed \textsf{DLADMC} achieves the (near-)oracle convergence rate. Numerical experiments are conducted to verify our conclusions.

\section*{Acknowledgment}
Weidong Liu's research is supported by National Program on Key Basic Research Project (973 Program, 2018AAA0100704), NSFC Grant No. 11825104 and 11690013, Youth Talent Support Program, and a grant from Australian Research Council. Xiaojun Mao's research is supported by Shanghai Sailing Program 19YF1402800. Raymond K.W. Wong's research is partially supported by the National Science Foundation under Grants DMS-1806063, DMS-1711952 (subcontract) and CCF-1934904.

\appendix

\section{Proofs}
\begin{proof}[Proof of Theorem \ref{thm:init}]
	As for the (i) in Theorem \ref{thm:init}, we obtain the upper bound directly from Theorem 4.6 of \citet{Alquier-Cottet-Lecue19}.
	
	As for (ii), by putting these $n_1n_2/(m_1m_2)$ estimators $\wh\bA_{\text{QMC},l}$ together, we focus on both the first and second term of the right hand side of the upper bound \eqref{eqn:QMCupper} respectively. It is easy to verify that the upper bound in the right hand side hold.
	
	In terms of the probability, we can conclude that
	\begin{align*}
	\sum_{l=1}^{l_1l_2}C_l\exp(-C_ls_lm_{\max}\log(m_{+}))\le\\
	\max\{C_l\}\exp(\log(n_1n_2)-\min\{C_l\}m_{\max}\log(m_{+})).
	\end{align*}
\end{proof}

\begin{prop}\label{appprop:f0}
	Suppose that Conditions (C1)-(C5) hold. Let  $h\geq c\log (n_{+})/N$ for some $c>0$ and $h=O((n_1n_2)^{-1/2}a_{N})$. We have 
	\[
	\Abs{\wh f\left(0\right)-f\left(0\right)}=O_{\pr}\left(\sqrt{\frac{\log (n_{+})}{Nh}}+\frac{a_{N}}{\sqrt{n_1n_2}}\right).
	\]
\end{prop}

\begin{proof}[Proof of Proposition \ref{appprop:f0}]
	Let 
	\begin{align*}
	D_{N,h}\left(\bA\right)=\frac{1}{Nh}\sum_{i=1}^{N}K\left(\frac{Y_{i}-\tr(\bX_{i}^{\rm T}\bA)}{h}\right).
	\end{align*}
	To prove the proposition, without loss of generality, we can assume that $\norm{\bA-\bA_{\star}}_{F}\leq a_{N}$. It follows that $\wh f(0)=D_{N,h}(\bA)$ and 	
	\[
	\Abs{\wh f\left(0\right)-f\left(0\right)}\le\underset{\norm{\bA-\bA_{\star}}_{F}\le a_{N}}{\sup}\Abs{D_{N,h}\left(\bA\right)-f\left(0\right)}.
	\]
	We denote $\bA_{\star}=(A_{\star,11},\dots,A_{\star,n_1n_2})$. For every $s$ and $t$, we divide the interval $[A_{\star,st}-a_{N},A_{\star,st}+a_{N}]$ into $(n_1n_2)^{M}$ small sub-intervals and each has length $2a_{N}/(n_1n_2)^{M}$, where $M$ is a large positive number. Therefore, there exists a set of matrices in $\mbR^{n_1\times n_2}$, $\{\bA_{(k)},1\le k\le s_{N}\}$ with $s_{N}\le (n_1n_2)^{M(n_1n_2)}$ and $\norm{\bA_{(k)}-\bA_{\star}}_{F}\leq a_{N}$, such that for any $\bA$ in the ball $\{\mathcal{A}:\bA\in\mbR^{n_1\times n_2},\norm{\bA-\bA_{\star}}_{F} \le a_{N}\}$, we have $\norm{\bA-\bA_{(k)}}_{F}\le 2\sqrt{n_1n_2}a_{N}/(n_1n_2)^{M}$ for some $1\le k\le s_{N}$. Therefore 
	\begin{align*}
	\Abs{\frac{1}{h}K\left(\frac{Y_{i}-\tr(\bX_{i}^{\rm T}\bA)}{h}\right)-\frac{1}{h}K\left(\frac{Y_{i}-\tr(\bX_{i}^{\rm T}\bA_{(k)})}{h}\right)}\le\\
	Ch^{-2}\Abs{\tr\{\bX_{i}^{\rm T}\left(\bA-\bA_{(k)}\right)\}}.
	\end{align*}
	This yields that
	\begin{align*}
	\underset{\norm{\bA-\bA_{\star}}_{F}\le a_{N}}{\sup}\Abs{D_{N,h}\left(\bA\right)-f\left(0\right)}-\\
	\underset{1\leq k\leq s_{N}}{\sup}\Abs{D_{N,h}\left(\bA_{(k)}\right)-f\left(0\right)}\le
	\frac{CN\sqrt{n_1n_2}a_{N}}{(n_1n_2)^{M+1}h^{2}}.
	\end{align*}
	By letting $M$ large enough, we have 
	\begin{align*}\label{appeqn:Ddiff}
	\underset{\abs{\bA-\bA_{\star}}_{2}\le a_{N}}{\sup}\Abs{D_{N,h}\left(\bA\right)-f\left(0\right)}-\\
	\underset{1\leq k\leq s_{N}}{\sup}\Abs{D_{N,h}\left(\bA_{(k)}\right)-f\left(0\right)}=O_{\pr}\left(n_{+}^{-\gamma}\right).
	\end{align*}
	It is enough to show that $\sup_{k}\abs{D_{N,h}(\bA_{(k)})-\mbE D_{N,h}(\bA_{(k)})}$ and $\sup_{k}\abs{\mbE D_{N,h}(\bA_{(k)})-f(0)}$ satisfy the bound in the lemma. Let $\mbE_{\ast}(\cdot)$ denote the conditional expectation given $\{\bX_{k}\}$.
	We have
	\begin{align*}
	\mbE_{\ast}\left\{\frac{1}{h}K\left(\frac{\epsilon_{i}-\tr\{\bX_{i}^{\rm T}\left(\bA-\bA_{\star}\right)\}}{h}\right)\right\}=\\
	\int_{-\infty}^{\infty}K\left(x\right)f\left\{hx+\tr\{\bX_{i}^{\rm T}\left(\bA-\bA_{\star}\right)\}\right\}dx\\
	=f\left(0\right)+O\left(h+\Abs{\tr\{\bX_{i}^{\rm T}\left(\bA-\bA_{\star}\right)\}}\right).
	\end{align*}
	Under Condition (C1), with the fact that $\mbE\abs{\tr\{\bX_{i}^{\rm T}(\bA-\bA_{\star})\}}\le(n_1n_2)^{-1}a_N$ and $\text{Var}\abs{\tr\{\bX_{i}^{\rm T}(\bA-\bA_{\star})\}}\le(n_1n_2)^{-1}a_N^2$, we have
	\begin{align*}
	\Abs{\mbE D_{N,h}\left(\bA_{(k)}\right)-f\left(0\right)}\le\\
	C\left(h+(n_1n_2)^{-1/2}\Norm{\bA_{(k)}-\bA_{\star}}_{F}\right)\\
	=O(h+(n_1n_2)^{-1/2}a_{N}).
	\end{align*}
	
	It remains to bound $\sup_{k}\abs{D_{N,h}(\bA_{(k)})-\mbE D_{N,h}(\bA_{(k)})}$. Put
	\[
	\xi_{i,k}=K\left(\frac{\epsilon_{i}-\tr\{\bX_{i}^{\rm T}\left(\bA_{(k)}-\bA_{\star}\right)\}}{h}\right).
	\]
	We have
	\begin{eqnarray*}
		\mbE_{\ast} \xi^{2}_{i,k}
		=\\
		h\int_{-\infty}^{\infty}\left\{K\left(x\right)\right\}^{2}f\left\{hx+\tr(\bX_{i}^{\rm T}\left(\bA_{(k)}-\bA_{\star}\right))\right\}dx
		\le Ch.
	\end{eqnarray*}
	Since $K(x)$ is bounded, we have 
	by the exponential inequality (Lemma 1 in \cite{Cai-Liu11}) and the fact that $\log (n_{+})=O(Nh)$, we have for any $\gamma>0$, there exists a constant $C>0$ such that
	\begin{align*}
	\sup_{k}\mbP\left(\left |\sum_{i=1}^{N}(\xi_{i,k}-\mbE \xi_{i,k}) \right |\ge C\sqrt{Nh\log (n_{+})}\right)\\
	= O \left(n_{+}^{-\gamma}\right).
	\end{align*}
	By letting $\gamma>M$, we can obtain that
	\begin{align*}
	\sup_{k}\Abs{D_{N,h}(\bA_{(k)})-\mbE D_{N,h}(\bA_{(k)})}=\\
	O_{\pr}\left(\sqrt{\frac{\log (n_{+})}{Nh}}\right).
	\end{align*}
	This completes the proof.
\end{proof}

\begin{lem}\label{applem:Xi} 
	We have for any $\gamma>0$, $|\bu|_{2}=1$ and $|\bv|_{2}=1$, there exists a constant $C>0$ such that
	\begin{eqnarray*}
		\Pr\left(\frac{1}{N}\sum_{i=1}^{N}\left(|\bv^{\rm T}\bX_i\bu|-\mbE|\bv^{\rm T}\bX_i\bu|\right)\geq C\sqrt{\frac{\log (n_{+})}{n_{\min}N}}\right)\\
		=O(n_{+}^{-\gamma}).
	\end{eqnarray*}
\end{lem}

\begin{proof}[Proof of Lemma \ref{applem:Xi}] 
	On one hand, we have $\mbE|\bv^{\rm T}\bX_i\bu|=O(n_{\min}^{-1})$. On the other hand, to apply Lemma 1 in \citet{Cai-Liu11}, we only need to find $B_{N}$ so that $\sum_{i}^{N}\mbE(|\bv^{\rm T}\bX_i\bu|^2\exp{\eta|\bv^{\rm T}\bX_i\bu|})\le B^2_{N}$. For each $i=1,\dots,N$, we have
	\begin{align*}
	\mbE(|\bv^{\rm T}\bX_i\bu|^2\exp{(\eta|\bv^{\rm T}\bX_i\bu|)})\\
	\le\frac{\overline{c}}{n_1n_2}\sum_{s=1}^{n_1}\sum_{t=1}^{n_2}u_{s}^2v_{t}^2\exp{(\eta\abs{u_{s}v_{t}})}\\
	\le\frac{\overline{c}}{n_1n_2}\sum_{s=1}^{n_1}\sum_{t=1}^{n_2}u_{s}^2v_{t}^2\exp{(\eta u_{s}^2)}\exp{(\eta v_{t}^2)}\\
	\le\frac{C(n_1+n_2)}{n_1n_2}=\frac{C}{n_{\min}}.
	\end{align*}
	Take $x^2=\gamma\log(n_{+})$ and $B^2_{N}=C\gamma^{-1}Nn_{\min}^{-1}$ in Lemma 1 of \citet{Cai-Liu11}, we can get the conclusion.
\end{proof}

Denote $\bB_{N}(\bA)\in\mbR^{n_1\times n_2}$ where
\begin{eqnarray}\label{appeqn:Bn}
B_{N}(\bA)=\frac{1}{N}\sum_{i=1}^{N}\left[\bX_{i}\Ind{\epsilon_{i}\le \tr\left\{\bX_{i}^{\rm T}\left(\bA-\bA_{\star}\right)\right\}}\right.\nonumber\\
\left.-\bX_{i}f\left( \tr\left\{\bX_{i}^{\rm T}\left(\bA-\bA_{\star}\right)\right\}\right)\right]\cr
-\frac{1}{N}\sum_{i=1}^{N}\left[\bX_{i}\Ind{\epsilon_{i}\le 0}-\bX_{i}f\left(0\right)\right].
\end{eqnarray}
Let $\bTheta=\{\bA: \norm{\bA-\bA_{\star}}_{F}\leq c\}$ for some $c>0$.

\begin{lem}\label{lem3} 
	We have for any $\gamma>0$, there exists a constant $C>0$ such that
	\begin{eqnarray*}
		\sup_{|\bv|_{2}=1}\sup_{|\bu|_{2}=1}\Pr\Big{(}\sup_{\bA\in \bTheta}\frac{\sqrt{n_1n_2}|\bv^{\rm T}B_{N}(\bA)\bu|}{\sqrt{\norm{\bA-\bA_{\star}}_{F}+n_{\max}\log (n_{+})/N}}\geq\\ C\sqrt{\frac{\log (n_{+})}{n_{\min}N}}\Big{)}
		=O(n_{+}^{-\gamma}).
	\end{eqnarray*}
\end{lem}

\begin{proof}[Proof of Lemma \ref{lem3}] 
	We define $\mbR^{n_1\times n_2}$, $\{\bA_{(k)},1\le k\le s_{N}\}$ as in the proof of Proposition \ref{appprop:f0} with by replacing $a_{N}$ with $c$. Then for any $\bA\in \bTheta$, there exists $\bA_{(k)}$ with $\norm{\bA-\bA_{(k)}}_{F}\leq 2c\sqrt{n_1n_2}/(n_1n_2)^{M}$ and we have
	\begin{eqnarray*}
		&&\Big{|}\frac{\sqrt{n_1n_2}|\bv^{\rm T}B_{N}(\bA)\bu|}{\sqrt{\norm{\bA-\bA_{\star}}_{F}+n_{\max}\log (n_{+})/N}}
		-\cr
		&&\frac{\sqrt{n_1n_2}|\bv^{\rm T}B_{N}(\bA_{(k)})\bu|}{\sqrt{\norm{\bA_{(k)}-\bA_{\star}}_{F}+n_{\max}\log (n_{+})/N}}\Big{|}\cr
		&&\leq 
		\Big{|}\frac{\sqrt{n_1n_2}|\bv^{\rm T}B_{N}(\bA_{(k)})\bu|}{\sqrt{\norm{\bA-\bA_{\star}}_{F}+n_{\max}\log (n_{+})/N}}
		-\cr
		&&\frac{\sqrt{n_1n_2}|\bv^{\rm T}B_{N}(\bA_{(k)})\bu|}{\sqrt{\norm{\bA_{(k)}-\bA_{\star}}_{F}+n_{\max}\log (n_{+})/N}}\Big{|}\cr
		&&\quad+\frac{\sqrt{n_1n_2}|\bv^{\rm T}B_{N}(\bA)\bu-\bv^{\rm T}B_{N}(\bA_{(k)})\bu|}{\sqrt{\norm{\bA-\bA_{\star}}_{F}+n_{\max}\log (n_{+})/N}}\cr
		&&=:I_{1}+I_{2}.
	\end{eqnarray*}
	It is easy to see that
	\begin{eqnarray*}
		|I_{1}|\leq C\frac{\sum_{i=1}^{N}|\bv^{\rm T} \bX_{i}\bu|\tr\{\bX_{i}^{\rm T}\left(\bA_{(k)}-\bA_{\star}\right)\}}{N}\times\\ \frac{\sqrt{n_1n_2}\times c\sqrt{n_1n_2}}{(n_1n_2)^{M}(c+n_{\max}\log (n_{+})/N)^{3/2}}=:I_{3}.
	\end{eqnarray*}
	With Lemma \ref{applem:Xi}, we can show that
	\begin{eqnarray*}
		\Pr\Big{(}I_{3}\geq C\sqrt{\frac{\log (n_{+})}{n_{\min}N}}\Big{)}=O(n_{+}^{-\gamma}),
	\end{eqnarray*}
	for any $\gamma>0$ by letting $M$ be sufficiently large. For $I_{2}$, noting that 
	\begin{align*}
	\Big{|}f\left(\tr\{\bX_{i}^{\rm T}\left(\bA-\bA_{\star}\right)\}\right)-f\left( \tr\{\bX_{i}^{\rm T}\left(\bA_{(k)}-\bA_{\star}\right)\}\right)\Big{|}\\
	\leq C\sqrt{n_1n_2}/(n_1n_2)^{M},
	\end{align*}
	we have
	\begin{eqnarray*}
		&&|I_{2}|\leq \sqrt{n_1n_2}\Big{(}\frac{cn_{\max}\log (n_{+})}{N}\Big{)}^{-1/4}\frac{1}{N}\sum_{i=1}^{N}|\bv^{\rm T}\bX_{i}\bu|\times\\
		&&\Ind{|\epsilon_{i}- \tr\{\bX_{i}^{\rm T}\left(\bA_{(k)}-\bA_{\star}\right)\}|\leq 2c\sqrt{n_1n_2}/(n_1n_2)^{M}}\\
		&&+C\frac{cn_1n_2}{(n_1n_2)^{M}}\Big{(}\frac{cn_{\max}\log (n_{+})}{N}\Big{)}^{-1/4}\times\\
		&&\frac{1}{N}\sum_{i=1}^{N}|\bv^{\rm T}\bX_{i}\bu|\\
		&&=:I_{4}+I_{5}.
	\end{eqnarray*}
	It is easy to show that $\mbE(I_{4})=o(\sqrt{\log (n_{+})/(n_{\min}N)})$ with $M$ large enough and
	\begin{eqnarray*}
		&&\Pr\Big{(}I_{5}\geq C\sqrt{\frac{\log (n_{+})}{n_{\min}N}}\Big{)}\leq \cr
		&&\sum_{i=1}^{N}\Pr\Big{(}|\bv^{\rm T}\bX_{i}\bu|\geq \frac{(n_1n_2)^{M-2}N^{1/4}}{n_{\max}\log (n_{+})}\Big{)}\cr
		&&=O(n_{+}^{-\gamma}),
	\end{eqnarray*}
	for any $\gamma>0$ by letting $M$ be sufficiently large. Also for some $\eta>0$,
	\begin{eqnarray*}
		&&\mbE(|\bv^{\rm T}\bX_{i}\bu|^{2}\exp(\eta |\bv^{\rm T}\bX_{i}\bu|)\times\cr&&\Ind{|\epsilon_{i}- \tr\{\bX_{i}^{\rm T}\left(\bA_{(k)}-\bA_{\star}\right)\}|\leq 2c\sqrt{n_1n_2}/(n_1n_2)^{M}})\cr
		&&
		\leq C\sqrt{n_1n_2}(n_1n_2)^{-M}\mbE |\bv^{\rm T}\bX_{i}\bu|^{2}\exp(\eta |\bv^{\rm T}\bX_{i}\bu|) \cr
		&&=O(1/((n_1n_2)^{M-1/2}n_{\min})).
	\end{eqnarray*}
	Now by the exponential inequality in \cite{Cai-Liu11} (taking $x=\sqrt{\gamma\log (n_{+})}$, $B_{n}=\sqrt{\gamma^{-1}N\log (n_{+})/n_{\min}}$ and noting that $1/((n_1n_2)^{M-1/2}n_{\min})=o(B^{2}_{N})$), we have for large $C>0$,
	\begin{eqnarray*}
		&&\Pr\Big{(}|I_{4}-\mbE (I_{4})|\geq C\sqrt{\frac{\log (n_{+})}{n_{\min}N}}\Big{)}\\
		&&=O(n_{+}^{-\gamma}).
	\end{eqnarray*}
	As $s_{N}\leq (n_1n_2)^{M(n_1n_2)}$, by choosing $C$ sufficiently large such that  $\gamma>M$, it is enough to show that for any $\gamma>0$,
	\begin{eqnarray}\label{di1}
	&&\sup_{|\bv|_{2}=1}\sup_{|\bu|_{2}=1}\max_{k}\Pr\Big{(}\sqrt{n_1n_2}|\bv^{\rm T}B_{N}(\bA_{(k)})\bu|\times\nonumber\\
	&&\frac{1}{\sqrt{\norm{\bA_{(k)}-\bA_{\star}}_{F}+n_{\max}\log (n_{+})/N}} \geq C\sqrt{\frac{\log (n_{+})}{n_{\min}N}}\Big{)}\nonumber\\
	&&=O(n_{+}^{-\gamma}).
	\end{eqnarray}
	Set
	\begin{eqnarray*}
		Z_{i}(\bA)=\Ind{\epsilon_{i}\le \tr\{\bX_{i}^{\rm T}\left(\bA-\bA_{\star}\right)\}}-f\left( \tr\{\bX_{i}^{\rm T}\left(\bA-\bA_{\star}\right)\}\right).
	\end{eqnarray*}
	Then we have
	\begin{eqnarray*}
		&&\mbE (\bv^{\rm T}\bX_{i}\bu)^{2}(Z_{i}(\bA)-Z_{i}(\bA_{\star}))^{2}\exp(\eta |\bv^{\rm T}\bX_{i}\bu|)\cr
		&&\leq C(n_1n_2)^{-1}\norm{\bA-\bA_{\star}}_{F}\times\cr
		&&\sup_{|\bv|_{2}=1,|\bu|_{2}=1}\mbE (\bv^{\rm T}\bX_{i}\bu)^{2}\exp(\eta |\bv^{\rm T}\bX_{i}\bu|)\cr
		&&\leq C(n_1n_2)^{-1}\norm{\bA-\bA_{\star}}_{F}n_{\min}^{-1}.
	\end{eqnarray*}
	Now letting $B^{2}_{N}=C\gamma^{-1}(N\norm{\bA_{(k)}-\bA_{\star}}_{F}/(n_1n_2)+N\log (n_{+})/n_{\min})$ and $x^{2}=\gamma \log (n_{+})$ in Lemma 1 in \cite{Cai-Liu11}, we can show (\ref{di1}) holds.
\end{proof}

Let 
\begin{eqnarray*}
	U_{N}=\underset{\Norm{\bA-\bA_{\star}}_{F}\le a_{N}}{\sup}\Norm{\bB_{N}(\bA)}.
\end{eqnarray*}
For a unit ball $B$ in $R^{s}$, we have the fact that there exist $q_{s}$ balls with centers $\bx_{1},\ldots,\bx_{q_{s}}$ and radius $z$ (i.e., $B_{i}=\{\bx\in R^{s}:|\bx-\bx_{i}|\leq z\}$, $1\leq i\leq q_{s}$) such that $B\subseteq \cup_{i=1}^{q_{s}}B_{i}$ and $q_{s}$ satisfies $q_{s}\leq (1+2/z)^{s}$. 
Then by a standard $\mathcal{E}-$net argument, for any matrix $\bA\in \mbR^{n_1\times n_2}$, there exist $\bv_{1},...,\bv_{b_{1}}$ and $\bu_{1},...,\bu_{b_{2}}$ (which do not depend on $\bA$)
with $|\bv_{i}|_{2}=1$ and $|\bu_{i}|_{2}=1$, $b_{1}\leq 9^{n_1}$ and $b_{2}\leq 9^{n_2}$ such that 
\begin{eqnarray}\label{appeqn:adss}
\|\bA\|\leq 5\max_{1\leq i\leq b_{1}}\max_{1\leq j\leq b_{2}}|\bv^{\rm T}_{i}\bA\bu_{j}|. 
\end{eqnarray}
So we have $U_{N}\leq 5\max_{1\leq i\leq b_{1}}\max_{1\leq j\leq b_{2}}\abs{\bv^{\rm T}_{i}B_{N}(\bA_{(k)})\bu_{j}}$. 
Assume the initial value $(n_1n_2)^{-1/2}\norm{\bA_{\star}-\wh\bA_{0}}_F=o_{\pr}(1)$. 
By Lemma \ref{lem3}, we have
\begin{eqnarray*}
	U_{N}=O_{\pr}\left(\sqrt{\frac{\Norm{\wh\bA_{0}-\bA_{\star}}_{F}\log (n_{+})}{n_1n_2n_{\min}N}}+\frac{\log (n_{+})}{n_{\min}N}\right).
\end{eqnarray*}
So we have the following lemma.

\begin{lem}\label{lem:Un} 
	Assume that Conditions (C1)-(C6) hold. We have
	\[
	U_{N}=O_{\pr}\Big{(}\sqrt{\frac{a_{N}\log (n_{+})}{n_1n_2n_{\min}N}}+\frac{\log (n_{+})}{n_{\min}N}\Big{)}.
	\]
\end{lem}

To obtain Theorem \ref{thm:keyt} which related to the repeated refinements, we consider the following one-step refinement result at first.
\begin{thm}[One-step refinement]\label{thm:key}
	Suppose that Conditions (C1)--(C5) hold and $\bA_{\star}\in\mathcal{B}(a,n_1,n_2)$. By choosing the bandwidth $h\asymp (n_1n_2)^{-1/2}a_{N}$ and taking 
	\begin{align*}
	\lambda_{N}=C\left(\sqrt{\frac{\log (n_{+})}{n_{\min}N}}+\frac{a_{N}^{2}}{n_{\min}(n_1n_2)}\right),
	\end{align*}
	where $C$ is a sufficient large constant, we have
	\begin{align}\label{appeqn:rmse}
	\frac{\Norm{\wh \bA^{(1)}-\bA_{\star}}_{F}^2}{n_1n_2}=
	O_{\pr}\left[\max\left\{\sqrt{\frac{\log(n_{+})}{N}},r_{\star}\left(\frac{n_{\max}\log (n_{+})}{N}+\frac{a_{N}^{4}}{n_{\min}^2(n_1n_2)}\right)\right\}\right].
	\end{align}
\end{thm}

To obtain Theorems \ref{thm:key} and \ref{thm:keyt}, we require Lemmas \ref{lem:error} and \ref{lem:errort} respectively. 

\begin{lem}\label{lem:error}
	Suppose that Conditions (C1)--(C5) hold and $\bA_{\star}\in\mathcal{B}(a,n_1,n_2)$. By choosing the bandwidth $h\asymp (n_1n_2)^{-1/2}a_{N}$, we have
	\begin{align*}
	\Norm{\frac{1}{N}\sum_{i=1}^{N}\xi_i^{(1)}\bX_i}=&O_{\pr}\left(\sqrt{\frac{\log (n_{+})}{n_{\min}N}}+\frac{a_{N}^{2}}{n_{\min}(n_1n_2)}\right).
	\end{align*}
\end{lem}

Lemma \ref{lem:error} obtains the upper bound for the stochastic error term that
appears in the first update iteration of the initial estimator $\wh\bA_{0}$ fulfill condition (C5). It is easy to verify that our initial estimator $\wh\bA_{\text{LADMC},0}$ proposed in section \ref{sec:BLADMC} satisfy condition (C5).

\begin{proof}[Proof of Lemma \ref{lem:error}]
	Denote $\bH_{N}(\bA)\in\mbR^{n_1\times n_2}$ where
	\begin{align*}\label{appeqn:Hn}
	H_{N}(\bA)=\\
	\frac{\wh f^{-1}\left(0\right)}{N}\sum_{i=1}^{N}\bX_{i}\left\{f\left[ \tr\left\{\bX_{i}^{\rm T}\left(\bA-\bA_{\star}\right)\right\}\right]-f\left(0\right)\right\}\\
	+\frac{1}{N}\sum_{i=1}^{N}\bX_{i}\tr\left\{\bX_{i}^{\rm T}\left(\bA-\bA_{\star}\right)\right\}.
	\end{align*} We have 
	\begin{align*}
	\Norm{\frac{1}{N}\sum_{i=1}^{N}\xi_i^{(1)}\bX_i}\le\\
	\left\Vert-\frac{\wh f^{-1}\left(0\right)}{N}\sum_{i=1}^{N}\bX_{i}\left(\Ind{Y_{i}\le \tr\{\bX_{i}^{\rm T}\wh\bA_{0}\}}-\tau\right)\right.\\
	\left.+\frac{1}{N}\sum_{i=1}^{N}\bX_{i}\tr\left\{\bX_i^{\rm T}\left(\wh\bA_{0}-\bA_{\star}\right)\right\}\right\Vert
	\le\\
	\Norm{\bH_{N}(\wh\bA_0)}+\Abs{\wh f^{-1}(0)}\Norm{\frac{1}{N}\sum_{i=1}^{N}\left[\bX_{i}\Ind{\epsilon_{i}\le 0}-\bX_{i}f\left(0\right)\right]}\\
	+\Abs{\wh f^{-1}\left(0\right)}U_{N}.
	\end{align*}
	By Proposition \ref{appprop:f0} and $(n_1n_2)^{1/2}\log (n_{+})=o(Na_{N})$, we have $\wh f(0)\geq c$ for some $c>0$ with probability tending to one. Therefore,
	for the last term, by Lemma \ref{lem:Un}, we have 
	\begin{align*}
	\abs{\wh f^{-1}(0)}U_{N}=O_{\pr}\left(\sqrt{\frac{a_{N}\log (n_{+})}{n_1n_2n_{\min}N}}+\frac{\log (n_{+})}{n_{\min}N}\right).
	\end{align*}
	For the second term of the right hand side, by (\ref{appeqn:adss}) and the exponential inequality in \cite{Cai-Liu11}, follow the same proof with Lemma \ref{applem:Xi}, we have 
	\begin{eqnarray*}
		\Abs{\wh f^{-1}(0)}\Norm{\frac{1}{N}\sum_{i=1}^{N}\bX_{i}\left[\Ind{\epsilon_{i}\le 0}-f\left(0\right)\right]}
		=O_{\pr}\left(\sqrt{\frac{\log (n_{+})}{n_{\min}N}}\right).
	\end{eqnarray*}
	
	By second order Taylor expansion, under condition (C1) we have,
	\begin{align*}
	\frac{\wh f^{-1}\left(0\right)}{N}\sum_{i=1}^{N}\bv^{\rm T}\bX_{i}\bu\left[f\left( \tr\{\bX_{i}^{\rm T}\left(\bA_{\star}-\wh\bA_{0}\right)\}\right)-f\left(0\right)\right]\\
	=\frac{\wh f^{-1}\left(0\right)f\left(0\right)}{N}\sum_{i=1}^{N}\bv^{\rm T}\bX_{i}\bu\tr\left\{\bX_{i}^{\rm T}\left(\bA_{\star}-\wh\bA_{0}\right)\right\}\\
	+O(1)\frac{\wh f^{-1}\left(0\right)}{N}\sum_{i=1}^{N}|\bv^{\rm T}\bX_{i}\bu|\left[\tr\left\{\bX_{i}^{\rm T}\left(\bA_{\star}-\wh\bA_{0}\right)\right\}\right]^2.
	\end{align*}
	Let $\bv_{1},...,\bv_{b_{1}}$ and $\bu_{1},...,\bu_{b_{2}}$ be defined as in the argument above Lemma \ref{lem:Un}. Together with Lemma \ref{applem:Xi},we have 
	\begin{align*}
	\Abs{\bv^{\rm T}_{k}\bH_{N}\left(\wh\bA_0\right)\bu_j}\le\Abs{\wh f^{-1}\left(0\right)f\left(0\right)-1}\times\\
	\Abs{\frac{1}{N}\sum_{i=1}^{N}\bv^{\rm T}_k\bX_{i}\bu_j\tr\left\{\bX_{i}^{\rm T}\left(\bA_{\star}-\wh\bA_{0}\right)\right\}}\\
	+C\wh f^{-1}\left(0\right)\frac{1}{N}\sum_{i=1}^{N}|\bv^{\rm T}_{k}\bX_{i}\bu_j|\left[\tr\left\{\bX_{i}^{\rm T}\left(\bA_{\star}-\wh\bA_{0}\right)\right\}\right]^2\\
	\le C\left(\sqrt{\frac{\log (n_{+})}{Nh}}+\frac{a_{N}}{\sqrt{n_1n_2}}\right)\frac{\Norm{\bA_{\star}-\wh\bA_{0}}_{F}}{n_{\min}\sqrt{n_1n_2}}\\
	+C\frac{1}{n_{\min}(n_1n_2)}\Norm{\bA_{\star}-\wh\bA_{0}}_{F}^{2}
	\end{align*}
	We can easily have 
	\begin{align*}
	\Norm{\frac{1}{N}\sum_{i=1}^{N}\xi_i^{(1)}\bX_i}=O_{\pr}\left(\sqrt{\frac{\log (n_{+})}{n_{\min}N}}+\sqrt{\frac{a_{N}\log (n_{+})}{n_1n_2n_{\min}N}}\right.\\
	\left.+a_{N}\sqrt{\frac{\log (n_{+})}{n_{\min}^2n_1n_2Nh}}+\frac{a_{N}^{2}}{n_{\min}(n_1n_2)}\right).
	\end{align*}
	The lemma is proved.
\end{proof}

Define the observation operator $\Omega:\mbR^{n_1\times n_2}\to \mbR^{N}$ as $(\Omega(\bA))_k=\inner{\bX_k}{\bA}$.

\begin{proof}[Proof of Theorem \ref{thm:key}]
	Due to the basic inequality, we have
	\begin{align*}
	\frac{1}{N}\sum_{k=1}^{N}\left(\wt Y_{k}^{(1)}-\tr(\bX_k^{\rm T}\wh\bA)\right)^2+\lambda_{N}\Norm{\wh\bA}_{\ast}\le\\
	\frac{1}{N}\sum_{k=1}^{N}\left(\wt Y_{k}^{(1)}-\tr(\bX_k^{\rm T}\bA_{\star})\right)^2+\lambda_{N}\Norm{\bA_{\star}}_{\ast},
	\end{align*}
	which implies
	\begin{align*}
	\frac{1}{N}\Norm{\Omega\left(\bA_{\star}-\wh\bA\right)}_{F}^2+\lambda_{N}\Norm{\wh\bA}_{\ast}\\
	\le2\Inner{\wh\bA-\bA_{\star}}{\bm{\Sigma}^{(1)}}+\lambda_{N}\Norm{\bA_{\star}}_{\ast}\\
	\le2\Norm{\bm{\Sigma}^{(1)}}\Norm{\wh\bA-\bA_{\star}}_{\ast}+\lambda_{N}\Norm{\bA_{\star}}_{\ast}.
	\end{align*}
	
	Together with Lemma \ref{lem:error} and follow the proof of Theorem 3 in \citet{Klopp14}, it complete the proof.
\end{proof}

\begin{proof}[Proof of Lemma \ref{lem:errort}]
	Replacing the tuning parameter $\lambda_{N}$ by $\lambda_{N,t}$, Lemma \ref{lem:errort} follows directly from the proof of Lemma \ref{lem:error}.
\end{proof}

\begin{proof}[Proof of Theorem \ref{thm:keyt}]
	Similar with the proof of Theorem \ref{thm:key}, together with the result in Lemma \ref{lem:errort} we complete the proof.
\end{proof}

\section{Experiments (Cont')}
\subsection{Synthetic Data (Cont')}\label{appsec:syn}
In the following, we tested the proposed method $\textsf{DLADMC}$ with the initial estimator synthetically generated by adding standard Gaussian noises ($\mathcal{N}$(0,1)) to the ground truth matrix $\bA_{\star}$ and reported all the results in Table \ref{apptab:sim}.

\begin{table}[h]
	\centering
	\scriptsize
	\caption{The average RMSEs, MAEs, estimated ranks and their standard
		errors (in parentheses) of modified $\textsf{DLADMC}$ over 500 simulations. The number in the first column within the parentheses represents $T$ in Algorithm \ref{alg:DLADMC}.}
	\label{apptab:sim}
	\begin{tabular}{r|r|r|r}
		\hline
		(T) & \text{RMSE} & \text{MAE} & \text{rank} \\
		\hline
		S1(4) & 0.6364 (0.0238) & 0.4826 (0.0232) & 63.74 (5.37)  \\ 
		\hline
		S2(5) & 0.8985 (0.0407) & 0.6738 (0.0404) & 67.59 (6.76) \\ 
		\hline
		S3(5) & 0.4460 (0.0080) & 0.3179 (0.0067)  & 43.07 (6.00) \\ 
		\hline
		S4(4) & 0.8522 (0.0203) & 0.6229 (0.0210)  & 45.21 (5.52) \\ 
		\hline
	\end{tabular}
\end{table}

\subsection{Real-World Data (Cont')}\label{appsec:real}
\subsubsection{Effect of Iteration Number}
To understand the effect of the iteration number, we ran 10 iterations and report all the details in Table \ref{apptab:realt}.
Briefly, the smallest and largest RMSEs among these iterations are (0.9226,0.9255), (0.9344,0.9381), (1.0486,1.0554) and (1.0512,1.0591) with respect to the 4 datasets in Section 4.2. Even with the worst RMSEs, we achieve a similar conclusion as shown in Section 4.2 of the paper.

\begin{table}[h]
	\centering
	\scriptsize
	\caption{The RMSEs, MAEs and estimated ranks of \textsf{DLADMC} with different iteration number under dimensions $n_1=739$ and $n_2=918$.}\label{apptab:realt}
	\begin{tabular}{r|r|r|r|r|r|r}
		\hline
		& t & 1 & 2 & 3 & 4 & 5 \\
		\hline
		\multirow{4}{*}{\textsf{RawA}} & \text{RMSE} & 0.9253 & 0.9253 & 0.9229 & 0.9252 & 0.9233 \\
		& \text{MAE} & 0.7241 & 0.7267 & 0.7224 & 0.7264 & 0.7230 \\
		& \text{rank} & 54 & 50 & 53 & 45 & 59 \\
		\hline
		\multirow{4}{*}{\textsf{RawB}} & \text{RMSE} & 0.9368 & 0.9381 & 0.9344 & 0.9373 & 0.9363 \\
		& \text{MAE} & 0.7315 & 0.7344 & 0.7291 & 0.7340 & 0.7310 \\
		& \text{rank} & 57 & 51 & 59 & 44 & 40 \\
		\hline
		\multirow{4}{*}{\textsf{OutA}} & \text{RMSE} & 1.0550 & 1.0543 & 1.0509 & 1.0549 & 1.0506 \\
		& \text{MAE} & 0.8659 & 0.8648 & 0.8609 & 0.8673 & 0.8595 \\
		& \text{rank} & 28 & 35 & 48 & 29 & 33 \\
		\hline
		\multirow{4}{*}{\textsf{OutB}} & \text{RMSE} & 1.0591 & 1.0569 & 1.0532 & 1.0583 & 1.0527 \\
		& \text{MAE} & 0.8707 & 0.8679 & 0.8632 & 0.8713 & 0.8627 \\
		& \text{rank} & 24 & 33 & 45 & 31 & 30 \\
		\hline
		& t & 6 & 7 & 8 & 9 & 10 \\
		\hline
		\multirow{4}{*}{\textsf{RawA}} & \text{RMSE} & 0.9253 & 0.9235 & 0.9250 & 0.9227 & 0.9255 \\
		& \text{MAE} & 0.7265 & 0.7233 & 0.7264 & 0.7219 & 0.7268 \\
		& \text{rank} & 41 & 41 & 45 & 55 & 44 \\
		\hline
		\multirow{4}{*}{\textsf{RawB}} & \text{RMSE} & 0.9362 & 0.9352 & 0.9369 & 0.9345 & 0.9370 \\
		& \text{MAE} & 0.7328 & 0.7300 & 0.7333 & 0.7292 & 0.7339 \\
		& \text{rank} & 49 & 51 & 46 & 58 & 44 \\
		\hline
		\multirow{4}{*}{\textsf{OutA}} & \text{RMSE} & 1.0544 & 1.0486 & 1.0553 & 1.0491 & 1.0554 \\
		& \text{MAE} & 0.8671 & 0.8568 & 0.8695 & 0.8569 & 0.8697 \\
		& \text{rank} & 31 & 38 & 35 & 40 & 33\\
		\hline
		\multirow{4}{*}{\textsf{OutB}} & \text{RMSE} & 1.0572 & 1.0521 & 1.0577 & 1.0512 & 1.0582 \\
		& \text{MAE} & 0.8699 & 0.8616 & 0.8706 & 0.8602 & 0.8716 \\
		& \text{rank} & 30 & 28 & 31 & 30 & 33 \\
		\hline
	\end{tabular}
\end{table}

\subsubsection{MovieLens-1M}
To further demonstrate the scalability of our proposed method, we tested various methods on a larger MovieLens-1M\footnote{\url{https://grouplens.org/datasets/movielens/1m/}} dataset. This data set consists of 1,000,209 movie ratings provided by 6040 viewers on approximate 3900 movies. The ratings also range from 1 to 5. To evaluate the performance of different methods, we keep one fifth of the data to be test set and remaining to be training set. We refer it to as \textsf{Raw}. Similar to \citet{Alquier-Cottet-Lecue19}, we added artificial outliers by randomly changing $20\%$ of ratings 
that are equal to $5$ in the train set to $1$ and constructed \textsf{Out}. To avoid rows and columns that contain too few observations,
we only keep the rows and columns with at least $40$ ratings.
The resulting target matrix $\bA_\star$ is of dimension $4290\times 2505$.
For the proposed \textsf{DLADMC}, we fix the iteration number to $10$. For the proposed \textsf{BLADMC}, to faster the speed, we split the data matrix so that the number of row subsets $l_1=4$ and number of column subsets $l_2=3$. To save times, the tunning parameters for all the methods were chosen by the one-fold validation. The RMSEs, MAEs, estimated ranks and the total computing time (in seconds) are reported in Table \ref{tab:real}. 
For a fair comparison, we recorded the time of each method in the experiment with the selected tuning parameter.

\begin{table}[h]
	\centering
	\scriptsize
	\caption{The RMSEs, MAEs and estimated ranks of \textsf{DLADMC}, \textsf{BLADMC}, \textsf{ACL} and \textsf{MHT} under dimensions $n_1=4290$ and $n_2=2505$.}\label{apptab:real1M}
	\begin{tabular}{r|r|r|r|r}
		\hline
		& & \textsf{DLADMC} & \textsf{BLADMC} & \textsf{MHT} \\
		\hline
		\multirow{4}{*}{\textsf{Raw}} & \text{RMSE} & 0.8632 & 0.9733 & 0.8520 \\
		& \text{MAE} & 0.6768 & 0.7865 & 0.6680 \\
		& \text{rank} & 111 & 1911 & 156 \\
		& $t$ & 19593.58 & 1203.45 & 2113.55 \\
		\hline
		\multirow{4}{*}{\textsf{Out}} & \text{RMSE} & 0.9161 & 0.9733 &  0.9757 \\
		& \text{MAE} & 0.7331 & 0.7865 & 0.8021 \\
		& \text{rank} & 125 & 1913 & 45 \\
		& $t$ & 14290.16 & 1076.69 & 1053.58 \\
		\hline
	\end{tabular}
\end{table}

As \textsf{ACL} is not scalable to large dimensions, we could not obtain the results of \textsf{ACL} within five times of the running time of the proposed \textsf{DLADMC}. It is noted that under the raw data \textsf{Raw}, the proposed \textsf{DLADMC} performed similarly as the least squares estimator \textsf{MHT}. \textsf{BLADMC} lost some efficiency due to the embarrassingly parallel computing. For the dataset with outliers, the proposed \textsf{DLADMC} performed better than \textsf{MHT}.
	
\bibliography{QMC}

@article{Davies93,
	Author = {Davies, Patrick L},
	Journal = {The Annals of statistics},
	Pages = {1843--1899},
	Publisher = {JSTOR},
	Title = {Aspects of robust linear regression},
	Year = {1993}}

@article{Chandrasekaran-Sanghavi-Parrilo11,
	Author = {Chandrasekaran, Venkat and Sanghavi, Sujay and Parrilo, Pablo A and Willsky, Alan S},
	Journal = {SIAM Journal on Optimization},
	Number = {2},
	Pages = {572--596},
	Publisher = {SIAM},
	Title = {Rank-sparsity incoherence for matrix decomposition},
	Volume = {21},
	Year = {2011}}

@article{Chen-Chi-Fan19,
	Author = {Chen, Yuxin and Chi, Yuejie and Fan, Jianqing and Ma, Cong and Yan, Yuling},
	Journal = {arXiv preprint arXiv:1902.07698},
	Title = {Noisy matrix completion: Understanding statistical guarantees for convex relaxation via nonconvex optimization},
	Year = {2019}}

@article{Chen-Fan-Ma20,
	Author = {Chen, Yuxin and Fan, Jianqing and Ma, Cong and Yan, Yuling},
	Journal = {arXiv preprint arXiv:2001.05484},
	Title = {Bridging Convex and Nonconvex Optimization in Robust PCA: Noise, Outliers, and Missing Data},
	Year = {2020}}

@article{Xia-Yuan19,
	Author = {Xia, Dong and Yuan, Ming},
	Journal = {arXiv preprint arXiv:1909.00116},
	Title = {Statistical Inferences of Linear Forms for Noisy Matrix Completion},
	Year = {2019}}

@inproceedings{Lafond15,
	Author = {Lafond, Jean},
	Booktitle = {Conference on Learning Theory},
	Pages = {1224--1243},
	Title = {Low rank matrix completion with exponential family noise},
	Year = {2015}}

@article{Elsener-Geer18,
	Author = {Elsener, Andreas and van de Geer, Sara},
	Journal = {The Annals of Statistics},
	Number = {6B},
	Pages = {3481--3509},
	Publisher = {Institute of Mathematical Statistics},
	Title = {Robust low-rank matrix estimation},
	Volume = {46},
	Year = {2018}}

@article{Negahban-Wainwright12,
	Author = {Negahban, Sahand and Wainwright, Martin J},
	Journal = {Journal of Machine Learning Research},
	Number = {1},
	Pages = {1665--1697},
	Publisher = {JMLR. org},
	Title = {Restricted Strong Convexity and Weighted Matrix Completion: Optimal Bounds with Noise},
	Volume = {13},
	Year = {2012}}

@article{Negahban-Wainwright11,
	Author = {Negahban, Sahand and Wainwright, Martin J},
	Journal = {The Annals of Statistics},
	Pages = {1069--1097},
	Publisher = {JSTOR},
	Title = {Estimation of (near) low-rank matrices with noise and high-dimensional scaling},
	Year = {2011}}

@article{Fan-Gong-Zhu19,
	Author = {Fan, Jianqing and Gong, Wenyan and Zhu, Ziwei},
	Journal = {Journal of Econometrics},
	Publisher = {Elsevier},
	Title = {Generalized high-dimensional trace regression via nuclear norm regularization},
	Year = {2019}}

@article{Bach08,
	Author = {Bach, Francis R},
	Journal = {Journal of Machine Learning Research},
	Number = {Jun},
	Pages = {1019--1048},
	Title = {Consistency of Trace Norm Minimization},
	Volume = {9},
	Year = {2008}}

@article{Keshavan-Montanari-Oh10,
	Author = {Keshavan, Raghunandan H and Montanari, Andrea and Oh, Sewoong},
	Journal = {Journal of Machine Learning Research},
	Mynotes = {Matrix Completion},
	Number = {2057--2078},
	Pages = {1},
	Title = {Matrix Completion from Noisy Entries},
	Volume = {11},
	Year = {2010}}

@inproceedings{Srebro-Rennie-Jaakkola05,
	Author = {Srebro, Nathan and Rennie, Jason and Jaakkola, Tommi S},
	Booktitle = {Advances in neural information processing systems},
	Pages = {1329--1336},
	Title = {Maximum-margin matrix factorization},
	Year = {2005}}

@inproceedings{Bennett-Lanning07,
	Author = {Bennett, James and Lanning, Stan},
	Booktitle = {Proceedings of KDD cup and workshop},
	Mynotes = {Matrix Completion},
	Pages = {35},
	Title = {The netflix prize},
	Volume = {2007},
	Year = {2007}}

@article{Li13,
	Author = {Li, Xiaodong},
	Journal = {Constructive Approximation},
	Number = {1},
	Pages = {73--99},
	Publisher = {Springer},
	Title = {Compressed sensing and matrix completion with constant proportion of corruptions},
	Volume = {37},
	Year = {2013}}

@article{Chen-Jalali-Sanghavi13,
	Author = {Chen, Yudong and Jalali, Ali and Sanghavi, Sujay and Caramanis, Constantine},
	Journal = {IEEE Transactions on Information Theory},
	Number = {7},
	Pages = {4324--4337},
	Publisher = {IEEE},
	Title = {Low-rank matrix recovery from errors and erasures},
	Volume = {59},
	Year = {2013}}

@book{huber2011robust,
	__Markedentry = {[xchen3:6]},
	Author = {Huber, Peter J},
	Date-Added = {2020-01-17 23:18:19 +0800},
	Date-Modified = {2020-01-17 23:18:19 +0800},
	Publisher = {Springer},
	Title = {Robust statistics},
	Year = {2011}}

@inproceedings{Chen-Xu-Caramanis11,
	Author = {Chen, Yudong and Xu, Huan and Caramanis, Constantine and Sanghavi, Sujay},
	Booktitle = {Proceedings of the 28th International Conference on Machine Learning (ICML-11)},
	Pages = {873--880},
	Title = {Robust matrix completion and corrupted columns},
	Year = {2011}}

@article{Gross11,
	Author = {Gross, David},
	Journal = {Information Theory, IEEE Transactions on},
	Number = {3},
	Pages = {1548--1566},
	Publisher = {IEEE},
	Title = {Recovering low-rank matrices from few coefficients in any basis},
	Volume = {57},
	Year = {2011}}

@article{Candes-Tao10,
	Author = {Cand{\`e}s, Emmanuel J and Tao, Terence},
	Journal = {Information Theory, IEEE Transactions on},
	Number = {5},
	Pages = {2053--2080},
	Publisher = {IEEE},
	Title = {The power of convex relaxation: Near-optimal matrix completion},
	Volume = {56},
	Year = {2010}}

@article{Rohde-Tsybakov11,
	Author = {Rohde, Angelika and Tsybakov, Alexandre B},
	Journal = {The Annals of Statistics},
	Number = {2},
	Pages = {887--930},
	Publisher = {Institute of Mathematical Statistics},
	Title = {Estimation of High-Dimensional Low-Rank Matrices},
	Volume = {39},
	Year = {2011}}

@article{Klopp-Lounici-Tsybakov17,
	Author = {Klopp, Olga and Lounici, Karim and Tsybakov, Alexandre B},
	Journal = {Probability Theory and Related Fields},
	Number = {1-2},
	Pages = {523--564},
	Publisher = {Springer},
	Title = {Robust matrix completion},
	Volume = {169},
	Year = {2017}}

@article{Wong-Lee17,
	Author = {Wong, Raymond K. W. and Lee, Thomas C. M.},
	Journal = {The Journal of Machine Learning Research},
	Number = {1},
	Pages = {5404--5428},
	Publisher = {JMLR. org},
	Title = {Matrix completion with noisy entries and outliers},
	Volume = {18},
	Year = {2017}}

@article{Candes-Li-Ma11,
	Author = {Cand{\`e}s, Emmanuel J and Li, Xiaodong and Ma, Yi and Wright, John},
	Journal = {Journal of the ACM (JACM)},
	Number = {3},
	Pages = {11},
	Publisher = {ACM},
	Title = {Robust principal component analysis?},
	Volume = {58},
	Year = {2011}}

@article{Cai-Liu11,
	Author = {Cai,T Tony and Liu, Weidong},
	Journal = {Journal of the American Statistical Association},
	Number = {494},
	Pages = {672--684},
	Publisher = {Taylor \& Francis},
	Title = {Adaptive thresholding for sparse covariance matrix estimation},
	Volume = {106},
	Year = {2011}}

@article{Koltchinskii-Lounici-Tsybakov11,
	Author = {Koltchinskii, Vladimir and Lounici, Karim and Tsybakov, Alexandre B},
	Journal = {The Annals of Statistics},
	Number = {5},
	Pages = {2302--2329},
	Publisher = {Institute of Mathematical Statistics},
	Title = {Nuclear-Norm Penalization and Optimal Rates for Noisy Low-Rank Matrix Completion},
	Volume = {39},
	Year = {2011}}

@article{Mackey-Talwalkar-Jordan15,
	Author = {Mackey, Lester and Talwalkar, Ameet and Jordan, Michael I},
	Journal = {The Journal of Machine Learning Research},
	Number = {1},
	Pages = {913--960},
	Publisher = {JMLR. org},
	Title = {Distributed matrix completion and robust factorization},
	Volume = {16},
	Year = {2015}}

@article{Mazumder-Hastie-Tibshirani10,
	Author = {Mazumder, Rahul and Hastie, Trevor and Tibshirani, Robert},
	Journal = {Journal of Machine Learning Research},
	Pages = {2287--2322},
	Publisher = {JMLR. org},
	Title = {Spectral Regularization Algorithms for Learning Large Incomplete Matrices},
	Volume = {11},
	Year = {2010}}

@article{Chen-Liu-Mao19,
	Author = {Chen, Xi and Liu, Weidong and Mao, Xiaojun and Yang, Zhuoyi},
	Journal = {arXiv preprint arXiv:1906.05741},
	Title = {Distributed High-dimensional Regression Under a Quantile Loss Function},
	Year = {2019}}

@article{Alquier-Cottet-Lecue19,
	Author = {Alquier, Pierre and Cottet, Vincent and Lecu{\'e}, Guillaume},
	Journal = {The Annals of Statistics},
	Number = {4},
	Pages = {2117--2144},
	Publisher = {Institute of Mathematical Statistics},
	Title = {Estimation bounds and sharp oracle inequalities of regularized procedures with Lipschitz loss functions},
	Volume = {47},
	Year = {2019}}

@article{Cai-Zhou16,
	Author = {Cai, T Tony and Zhou, Wen-Xin},
	Journal = {Electronic Journal of Statistics},
	Number = {1},
	Pages = {1493--1525},
	Publisher = {The Institute of Mathematical Statistics and the Bernoulli Society},
	Title = {Matrix Completion via Max-Norm Constrained Optimization},
	Volume = {10},
	Year = {2016}}

@article{Klopp14,
	Author = {Klopp, Olga},
	Journal = {Bernoulli},
	Number = {1},
	Pages = {282--303},
	Publisher = {Bernoulli Society for Mathematical Statistics and Probability},
	Title = {Noisy Low-Rank Matrix Completion with General Sampling Distribution},
	Volume = {20},
	Year = {2014}}

@article{Candes-Plan10,
	Author = {Cand{\`e}s, Emmanuel J and Plan, Yaniv},
	Journal = {Proceedings of the IEEE},
	Number = {6},
	Pages = {925--936},
	Publisher = {IEEE},
	Title = {Matrix Completion with Noise},
	Volume = {98},
	Year = {2010}}

@article{Candes-Recht09,
	Author = {Cand{\`e}s, Emmanuel J and Recht, Benjamin},
	Journal = {Foundations of Computational Mathematics},
	Number = {6},
	Pages = {717--772},
	Publisher = {Springer},
	Title = {Exact Matrix Completion via Convex Optimization},
	Volume = {9},
	Year = {2009}}
\bibliographystyle{dcu}

\end{document}